%% file: trend_rnn_miot_drigout.tex
\documentclass{article}
\usepackage{amssymb}
\usepackage{amsmath}
\usepackage{amsthm}
\usepackage{caption}
\usepackage{subcaption}
\usepackage{graphicx}
\usepackage{dsfont}
\usepackage{stmaryrd}
\graphicspath{{figures/}{./other_figures/}}
\usepackage{url}
\usepackage{authblk}
\usepackage{csvsimple}
\usepackage{rotating}
\usepackage{hyperref}
\hypersetup{
	colorlinks,
	citecolor=black,
	filecolor=black,
	linkcolor=black,
	urlcolor=black
}
\usepackage[title]{appendix}
\usepackage[table]{xcolor}

\newcommand{\R}{\mathbb{R}}
\newcommand{\N}{\mathbb{N}}

\newcommand{\OU}{Ornstein-Uhlenbeck}
\newcommand{\IQR}{interquartile range}
\newcommand{\SP}{S\&P 500}

\newcommand*{\tran}{^{\mkern-1.5mu\mathsf{T}}}
\DeclareMathOperator{\sgn}{sgn}
\DeclareMathOperator{\E}{\mathbb{E}}

\DeclareMathOperator{\1}{\mathds{1}}
\newtheorem{theorem}{Theorem}

\title{An empirical study of neural networks for trend detection in time series}
\date{\today}
\author[1]{Miot, Alexandre\thanks{\texttt{\href{mailto:alexandre.miot@sgcib.com}{alexandre.miot@sgcib.com}}}}
\author[2]{Drigout, Gilles\thanks{\texttt{\href{mailto:gilles.drigout@sgcib.com}{gilles.drigout@sgcib.com}}}\thanks{The authors would like to thank Lionel Massoulard and Sandrine Ungari for their comments and fruitful discussions}}

\affil[1,2]{Soci\'et\'e G\'en\'erale CIB}

\begin{document}
\maketitle

\begin{abstract}
Detecting structure in noisy time series is a difficult task. One intuitive feature, which is of particular interest in financial applications, is the notion of trend.
From theoretical hints and using simulated time series, we empirically investigate the efficiency of standard recurrent neural networks (RNNs) to detect trends. We show the overall superiority and versatility of certain standard RNNs structures over various other estimators.
These RNNs could be used as basic blocks to build more complex time series trend estimators.
\end{abstract}

\input{tex/intro}
\clearpage
\input{tex/section_simulation_ts}
\pagebreak
\input{tex/section_rnn}

\pagebreak
\input{tex/section_non_model_based}

\pagebreak
\input{tex/section_model_based}
\pagebreak
\input{tex/conclusion}
\clearpage

\begin{appendices}
\appendix
\input{tex/annex_mletheory}

\clearpage
\input{tex/annex_convexcell}

\clearpage
\input{tex/annex_hidden_state}
\clearpage
\input{tex/annex_rnn_training_details}

\end{appendices}

\clearpage
\bibliographystyle{plain}
\bibliography{trend_rnn_miot_drigout}

\end{document}

%% file: tex/intro.tex
\section{Introduction}
When looking at any dataset, human brain is wired to detect patterns \cite{Kahneman2012}.
Time series are no exception and quite naturally we see ``trends'' when shown a plot of share prices. Trends seem a relevant feature of any forecasting mechanism for time series.
In this article, we focus on univariate time series having a conspicuous trend component as commonly found in financial data. Trending time series are not unique to finance and our work extends to other domains.
The main contributions of this article are:
\begin{itemize}
\item[-] Framing the problem into a classification problem emphasizing the usefulness of simulated data
\item[-] Building a general trend estimator for a wide range of dynamics
\item[-] Showing in a simple case why RNNs are good trend estimators
\item[-] Showing empirically the superiority of RNNs over standard estimators
\item[-] Deriving theoretical maximum likelihood estimators for the considered dynamics
\end{itemize}

We first describe our general framework establishing trend detection as a sequence to sequence classification problem. We then define the time series dynamics used in our simulations. Next, we explore the use of recurrent neural networks to detect trends. Thereupon, we empirically compare performance of standard RNNs structures. We then build a general purpose trend estimator called RNN baseline. We benchmark its performance against other estimators like convolutional networks. Finally, we compare its performance against estimators based on parameter estimation (MLE) of the modelled dynamic. Mathematical topics and detailed results have been left aside in the appendix.

%% file: tex/section_simulation_ts.tex
\section{Framework and data set}\label{section:timeseries}
In this section we define our framework, which basically tries to address the question: what setup should one consider to find a ``good'' general purpose estimator of trend in time series ?

\subsection{The thought process}\label{subsection:framework}
Trends can be interpreted as the slopes of a smooth function around which the time series oscillates. The simplest, and probably the closest to human intuition, would be to use piecewise linear functions as in described in \cite{Kim2009}. The issue with these filtering approaches is that they tend to be good ex post but slow to detect changes of trends. This is a real problem when the whole time series is not known in advance.
\\
We take a slightly different approach. If the future value of the time series is expected to be higher [respectively lower, equal] than the current one, then the time series is said to be trending up [respectively trending down, not trending].
At each time step, we assign a unique trend value noted $\delta_t$, the time-series is:
\begin{itemize}
	\item[-] trending downward at $t$ if $\delta_t < 0$
	\item[-] not trending at $t$ if $\delta_t = 0$
	\item[-] trending upward at $t$ if $\delta_t > 0$
\end{itemize}
We can directly translate this intuition into mathematical terms.
Consider a process $\{Y_t\}_{t>0}$ adapted to a filtration $\{\mathcal{F}_t\}_{t \geq 0}$, under some technical conditions, the Doob-Meyer theorem applies and $\{Y_t\}_{t>0}$ can be decomposed in an unique way as 
$$
\forall t\in[0, T],\quad Y_t = A_t + M_t 
$$
where $\{A_t\}_{t>0}$ is a predictable increasing [respectively decreasing, zero] process if $Y_t$ is a sub-martingale [respectively super-martingale, martingale] starting at 0 and $\{M_t\}_{t>0}$ is a martingale.
Obviously, we can map our intuitive definition to more precise concepts.\\$\{Y_t\}_{t>0}$ is:
\begin{tabular}{rcl}
	trending downward & $\Longleftrightarrow$ & $\{A_t\}_{t>0}$ is decreasing\\ 
	not trending & $\Longleftrightarrow$ & $\{A_t\}_{t>0}$ is null\\ 
	trending upward & $\Longleftrightarrow$ &  $\{A_t\}_{t>0}$ is increasing\\ 
\end{tabular}
\\
The monotonicity of the $\{A_t\}_{t>0}$ process will be our definition of the trend of $\{Y_t\}_{t>0}$ and thus a classification task with three labels $\{-1, 0, 1\}$ for downward, flat and upward trend.
Considering an It\^o process $\{Y_t\}_{t>0}$
$$
dY_t = \beta(t, Y_t) dt + \sigma(t, Y_t) dW_t
$$
where $\{W_t\}_{t>0}$ is a Wiener process. We can track the changing monotonicity of $A_t = \int_{0}^{t} \beta(s, Y_s) ds$ via the sign of $\beta(t, Y_t)$ which will be our practical definition of trend.
\\
The challenge at hand is to build an estimator of the sign of $\beta(t, Y_t)$, which will be our classification label. 
In the following, we will consider various time series dynamics where we control the sign of $\beta(t, Y_t)$. This gives us a framework to analyse the performance of various estimators, while controlling for the statistical properties of the dataset.\\
The classification task relies on the labelling of the training set. When using historical data, labelling is not easy to do: the definition of trend is subjective and usually depends on the choice of a time window or of a performance criterion.
On the contrary, when using simulated data, labelling of the training set is easy.
A general-purpose estimator of trend in a simulated environment is a useful building block for handling more complex real-life cases where no trend labels are available. It gives us a robust starting point on which we can build on\footnote{specializing using using transfer-learning for example}.



\subsection{Time series dynamics}\label{subsection:tsdynamics}
Our idea is to generate as many realistic datasets as possible, and to train trend estimators on those datasets.
If we train our estimator on a dataset rich enough to capture all the possible scenarios, we can hope to have an estimator robust to real-life conditions.
In the following, we consider three different types of dynamics, hopefully rich and diverse enough to match a lot of the real-life behaviour:
\begin{itemize}
	\item[-] a noisy piecewise linear process
	\item[-] a piecewise \OU\space process \cite{Uhlenbeck1930}
	\item[-] a Markovian switching process \cite{Hamilton1989}
\end{itemize}
The first two are piecewise meaning that we divide time into intervals on which the time series follows the chosen dynamic. A simple continuity constraint is applied to ``glue'' together these different periods.\\
In the rest of the section we define:
\begin{itemize}
	\item[-] a time interval $[0, T]$
	\item[-] for piecewise processes, a number $N$ of intervals $[t_i, t_{i+1}],\, i\in\llbracket 1, N\rrbracket$ of possibly different lengths
\end{itemize}

\subsubsection{Noisy Line Process}\label{subsection:noisyline}
We define a Noisy Line Process\footnote{or Piecewise Noisy Line} by a process $\{Y_t\}_{t\in[0, T]}$ for which
$$
\forall i \in \llbracket 1, N \rrbracket,\forall t \in ]t_i, t_{i+1}], \quad Y_t = Y_{t_i} + \mu_i (t-t_i)+ \sigma_i \epsilon_t
$$
where 
\begin{itemize}
	\item[-]$\mu_i$ is a slope parameter randomly chosen in $\left\{ -\gamma , \ldots , \frac{-\gamma}{n}, 0, \frac{\gamma}{n} , \ldots , \gamma\right\}$, where $\gamma>0$ is the maximum slope and $n\in\N^*$
	\item[-]$\sigma_i>0$ is a noise parameter
	\item[-]$\{\epsilon_t\}_{t\geq 0}$ are i.i.d. normal variables
\end{itemize}
The trend here is given by the sign of $\mu_i$. Figure \ref{fig:stylized_trajectories} displays some possible trajectories.

\begin{figure}[h]
	\centering
	\begin{subfigure}[b]{0.4\columnwidth}
		\centering\includegraphics[width=0.9\columnwidth]{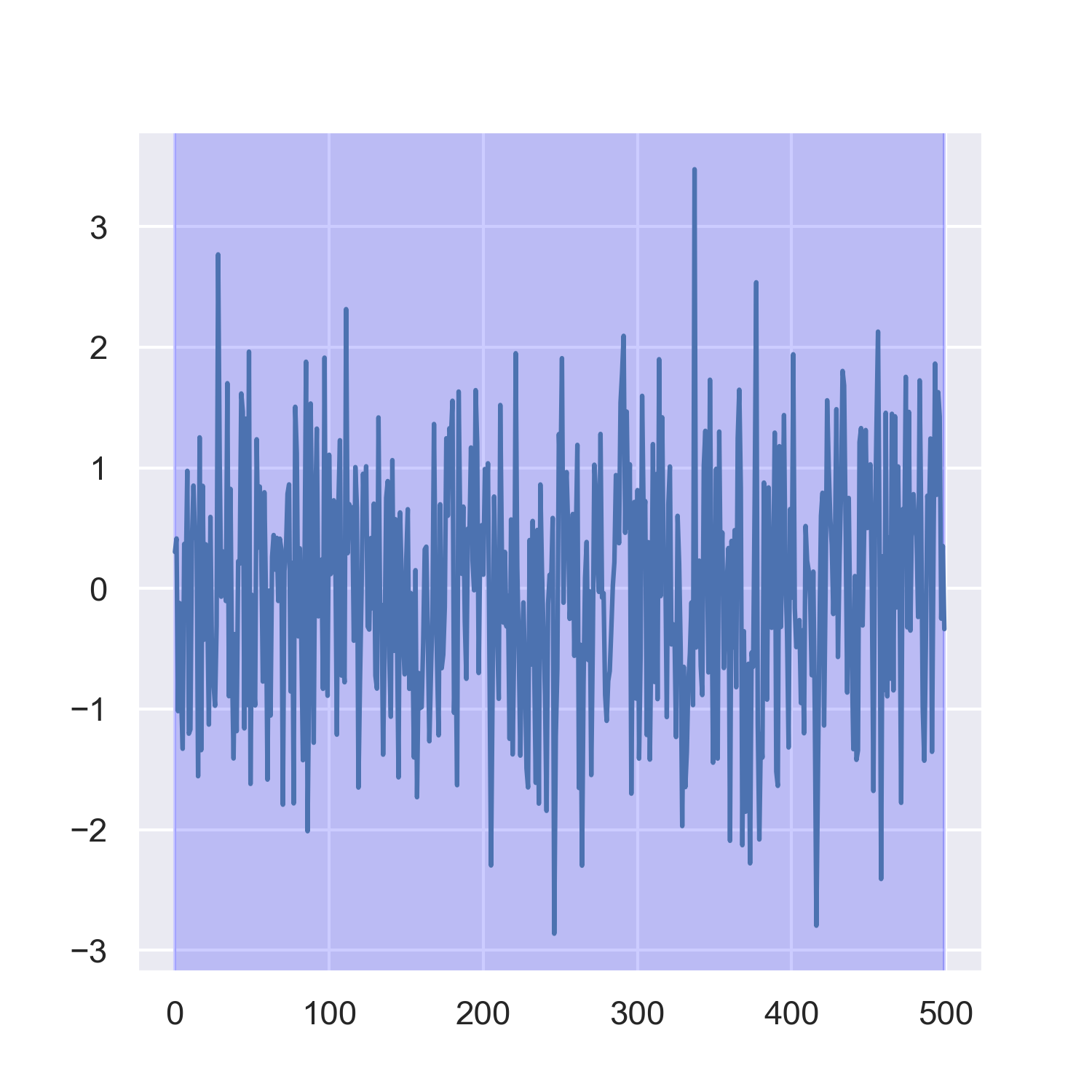}
		\caption{Flat process}
	\end{subfigure}
	~
	\begin{subfigure}[b]{0.4\columnwidth}
		\centering\includegraphics[width=0.9\columnwidth]{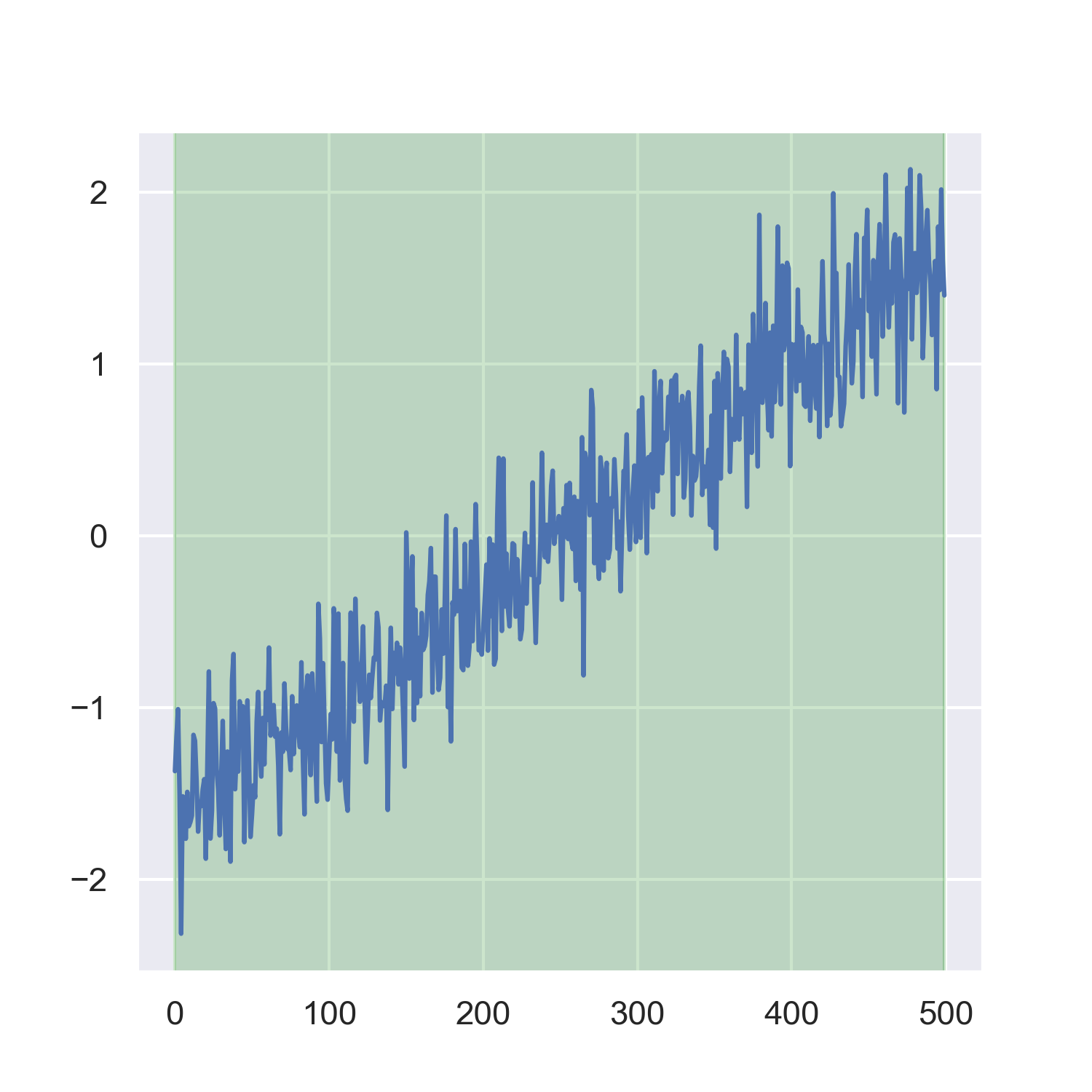}
		\caption{Trending up process}
	\end{subfigure}
	\\
	\begin{subfigure}[b]{0.4\columnwidth}
		\centering\includegraphics[width=0.9\columnwidth]{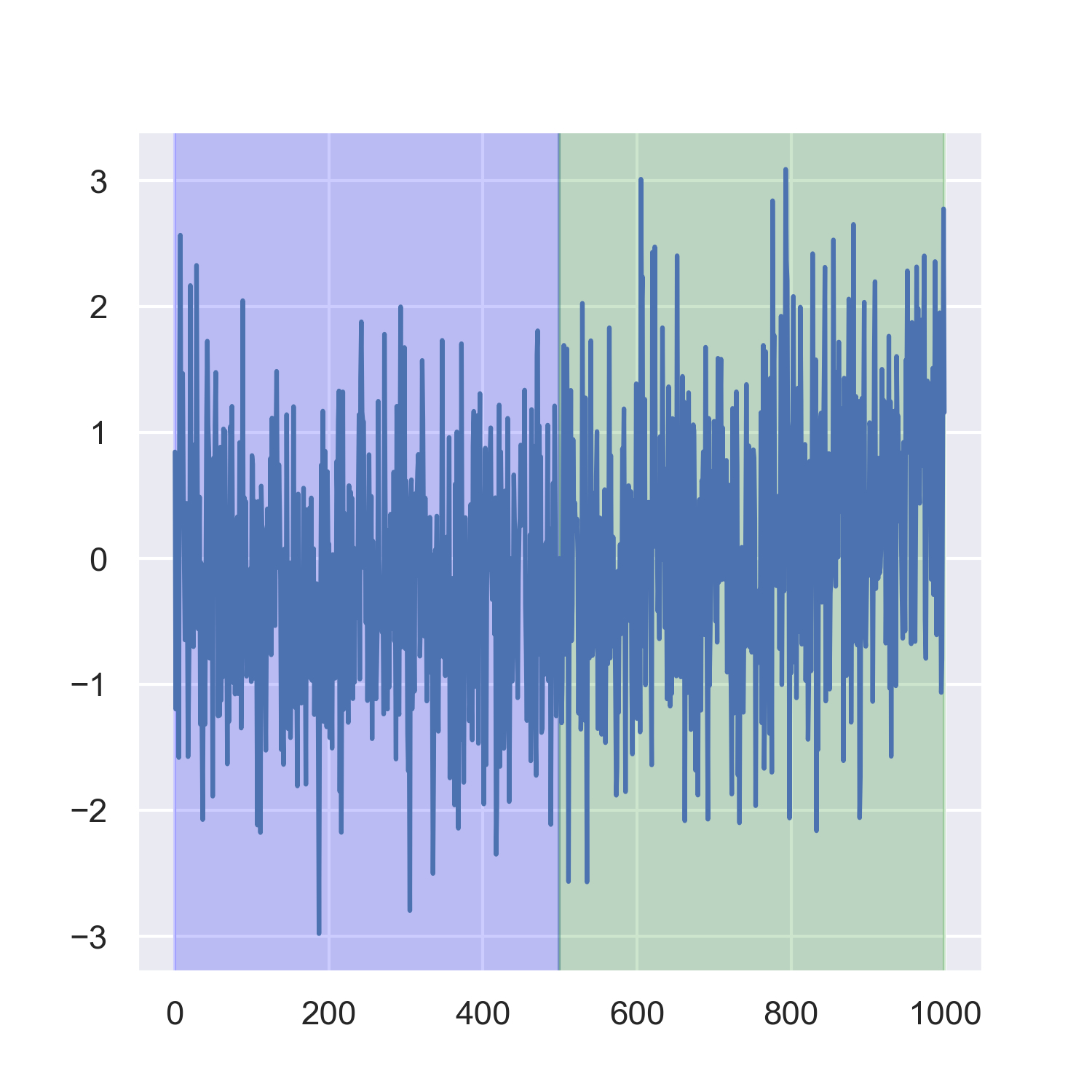}
		\caption{Two periods but very noisy}
	\end{subfigure}
	~
	\begin{subfigure}[b]{0.4\columnwidth}
		\centering\includegraphics[width=0.9\columnwidth]{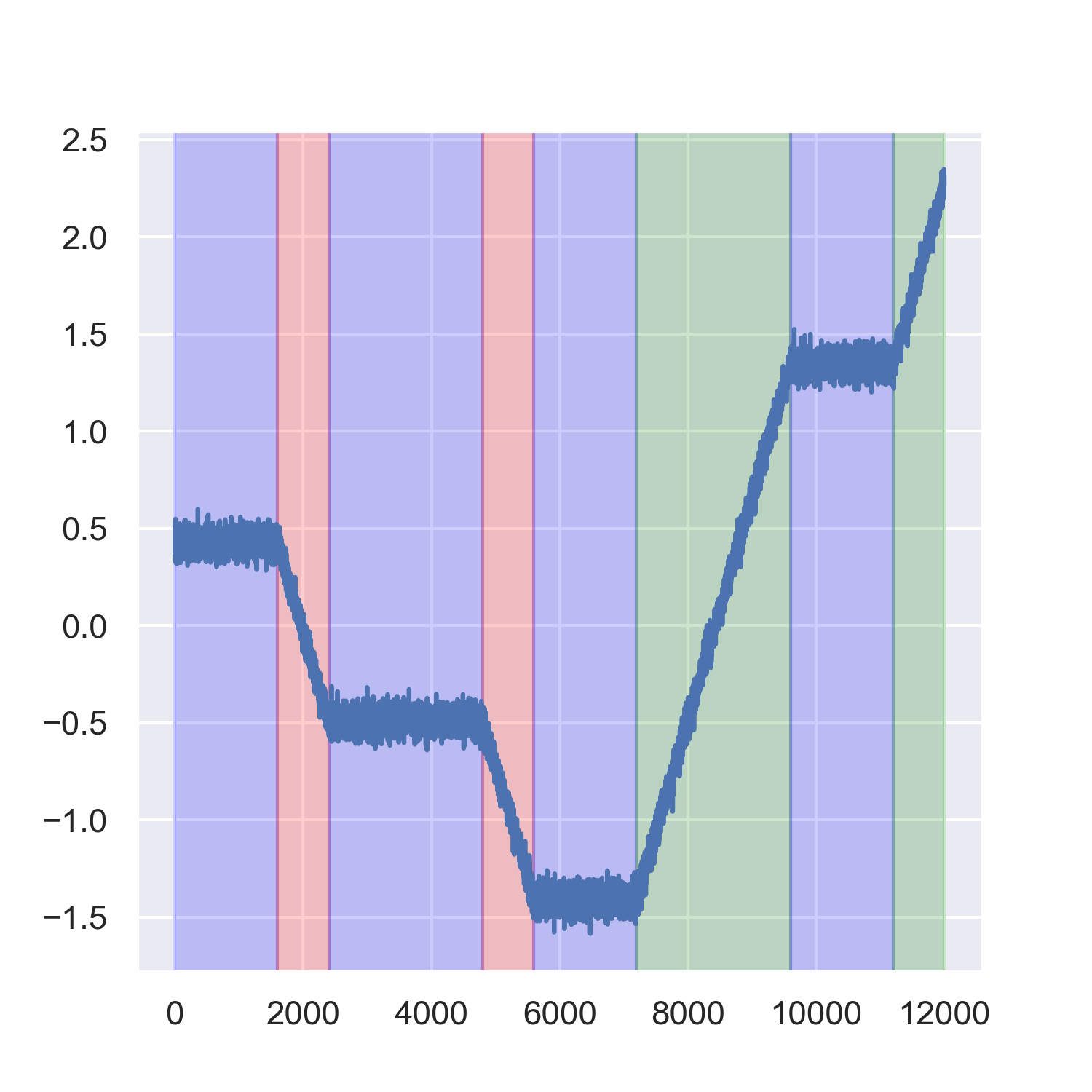}
		\caption{Several periods with less noise} 
	\end{subfigure}
	\caption{Noisy line process samples. Up in green, down in red and flat in blue}
	\label{fig:stylized_trajectories}
\end{figure}
\subsubsection{Piecewise \OU\space dynamic}\label{subsubsection:ou}
We define a Piecewise \OU\space Process as a process $\{Y_t\}_{t\in[0, T]}$ such that
\begin{multline*}
\forall i \in \llbracket 1, N \rrbracket,\forall t \in [t_i, t_{i+1}], \\
\quad Y_t = Y_{t_i} e^{-a_i(t-t_i)} + Y^\infty_i\,\left( 1-e^{-a_i(t-t_i)}\right) + \frac{\sigma}{\sqrt{2\,a_i}} W( e^{2a_i(t-t_i)}-1)\,  e^{-a_i(t-t_i)}
\end{multline*}
where $Y^\infty_i = \frac{\mu_i}{a_i}$ and $a_i\, ,\,\mu_i \geq 0$. If the intervals are big enough, $Y^\infty_i \approx Y_{t_{i+1}}$, and the trend label $l_i$ will be determined by
\begin{equation}
\frac{Y^\infty_i}{Y_{t_i}}\begin{cases}
>1, \quad l_i = +1 &\text{up trend}\\
=1, \quad l_i = 0 &\text{no trend} \\
<1, \quad l_i = -1 &\text{down trend}
\end{cases}
\end{equation}
Samples of piecewise \OU process are shown on figure \ref{fig:ou_trajectories}.

\begin{figure}[h!]
	\centering
	\begin{subfigure}[b]{0.4\columnwidth}
		\centering\includegraphics[width=0.9\columnwidth]{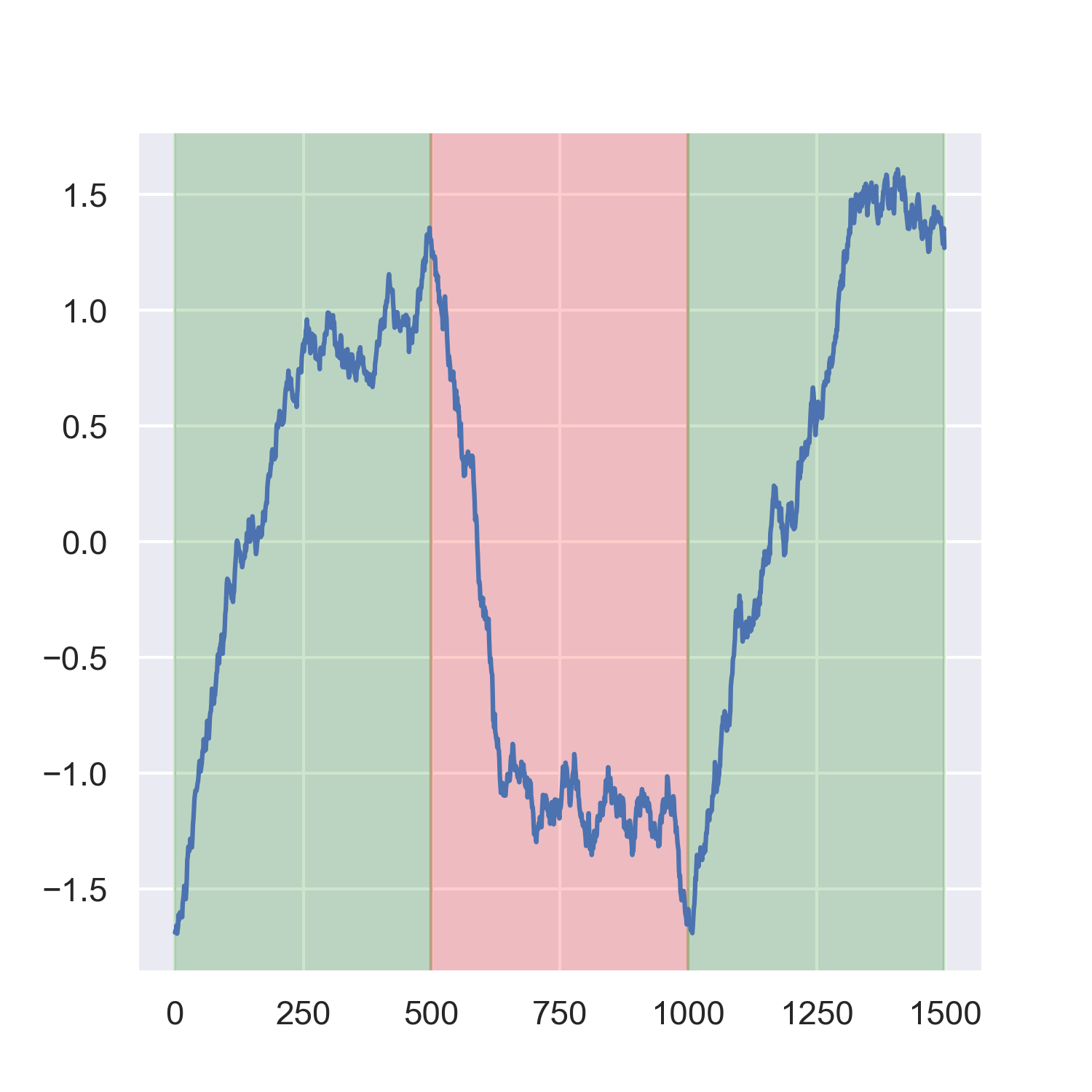}
		\caption{Three periods \OU\space process with weak ``pull''}
	\end{subfigure}
	~
	\begin{subfigure}[b]{0.4\textwidth}
		\centering\includegraphics[width=0.9\columnwidth]{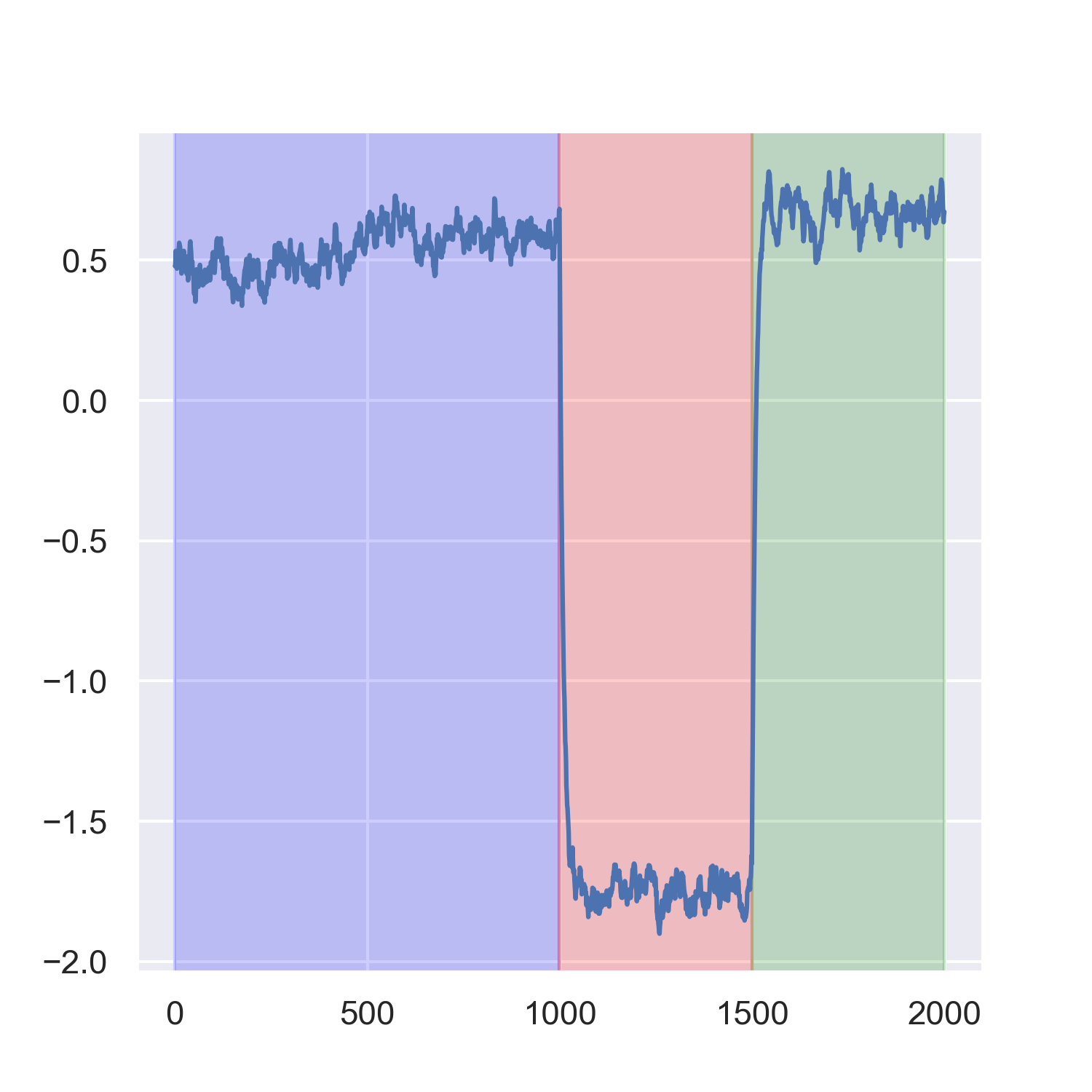}
		\caption{Four periods \OU\space process with strong ``pull''}
	\end{subfigure}
	\caption{Piecewise \OU\space processes. Up in green, down in red and flat in blue}
	\label{fig:ou_trajectories}
\end{figure}

\subsubsection{Switching Markovian dynamic}\label{subsubsection:switchingmarkovian}
The trend is given by a Markov chain $\{l_t\}_{t\geq 0}$ on finite states $\{-1, 0, +1\}$. The process $\{Y_t\}_{t\in[0, T]}$ is defined by
$$
Y_t = Y_0\exp\left(\sum_{i=1}^{t} \gamma_i \, \l_i + \sigma_t \epsilon_t \right)
$$
where $\{\gamma_t\}_{t\in[1, T]}$ is a slope process, $\{\sigma_t\}_{t\in[1, T]}$ a positive noise process and $\{\epsilon_t\}_{t\in[1, T]}\sim \mathcal{N}(0, 1)$. In practice, $\{\gamma_t\}_{t\in[1, T]}$ and $\{\sigma_t\}_{t\in[1, T]}$ are constant with time, the constant being randomly chosen in a discrete distribution.
This process exhibits a rich set of trajectories as seen on figure \ref{fig:markovian_trajectories}.

\begin{figure}[h]
	\centering
	\begin{subfigure}[b]{0.4\textwidth}
		\centering\includegraphics[width=0.9\columnwidth]{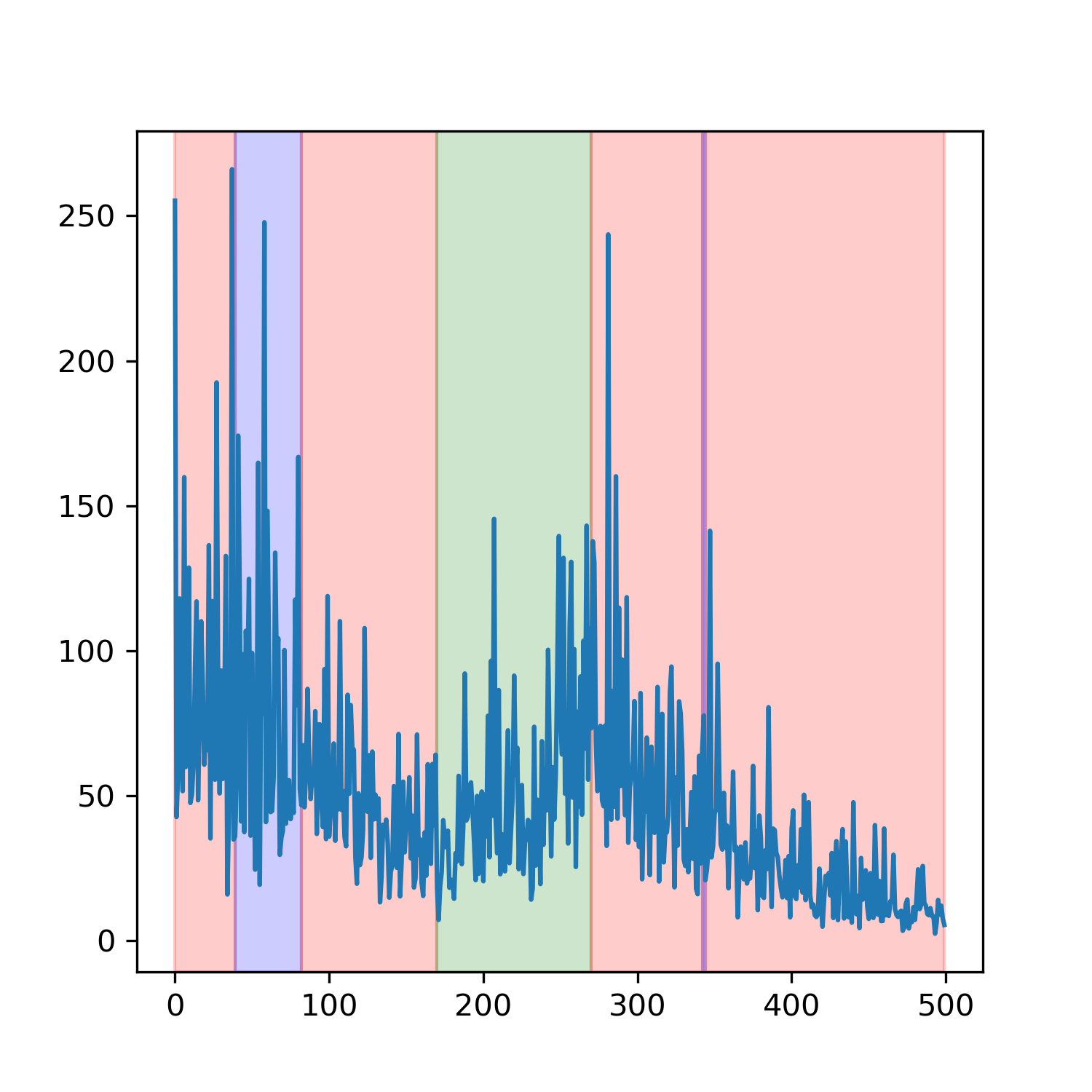}
		\caption{Trendy process with noise}
	\end{subfigure}
	~
	\begin{subfigure}[b]{0.4\textwidth}
			\centering\includegraphics[width=0.9\columnwidth]{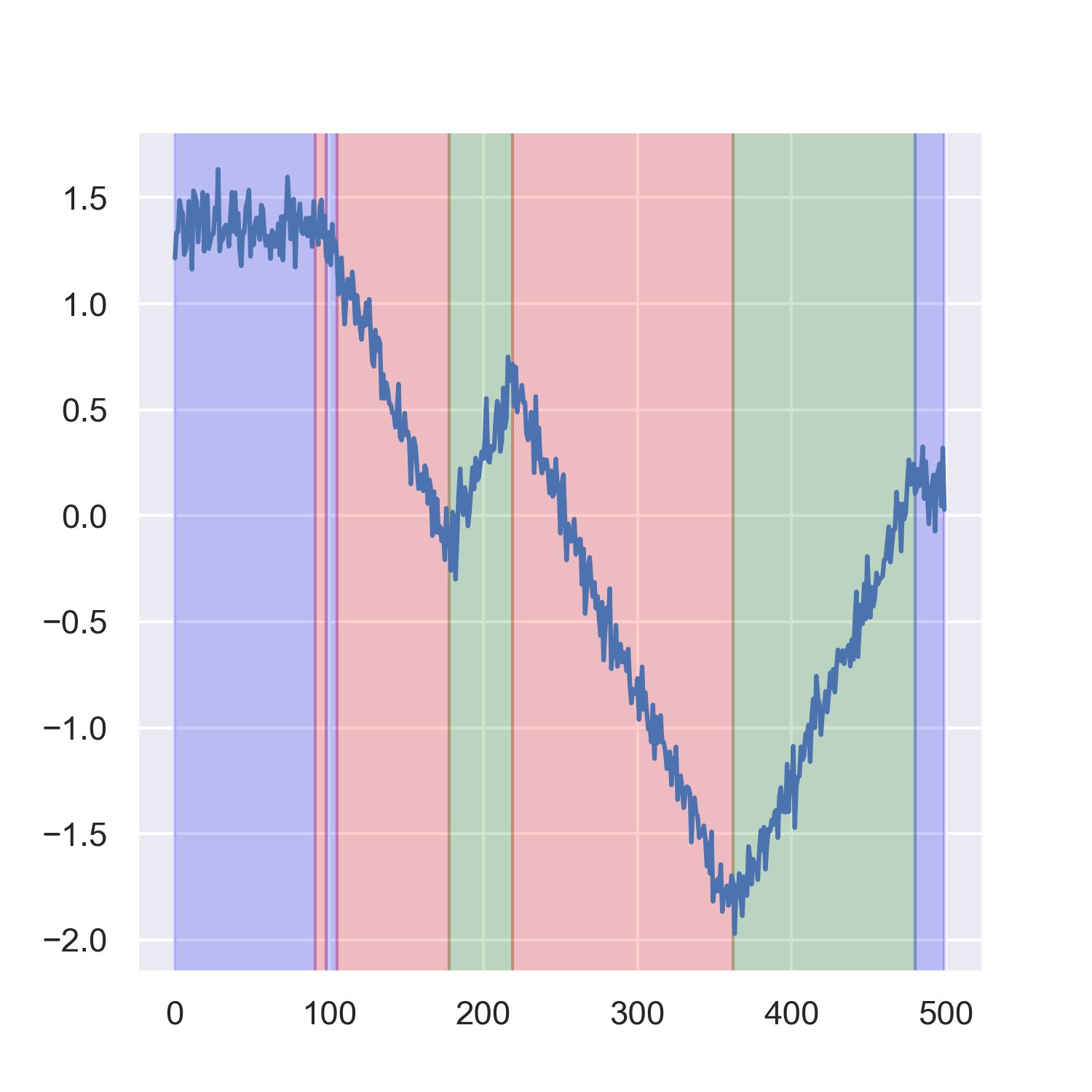}
			\caption{Trendy process with low noise}
	\end{subfigure}
	\\
	\begin{subfigure}[b]{0.4\textwidth}
		\centering\includegraphics[width=0.9\columnwidth]{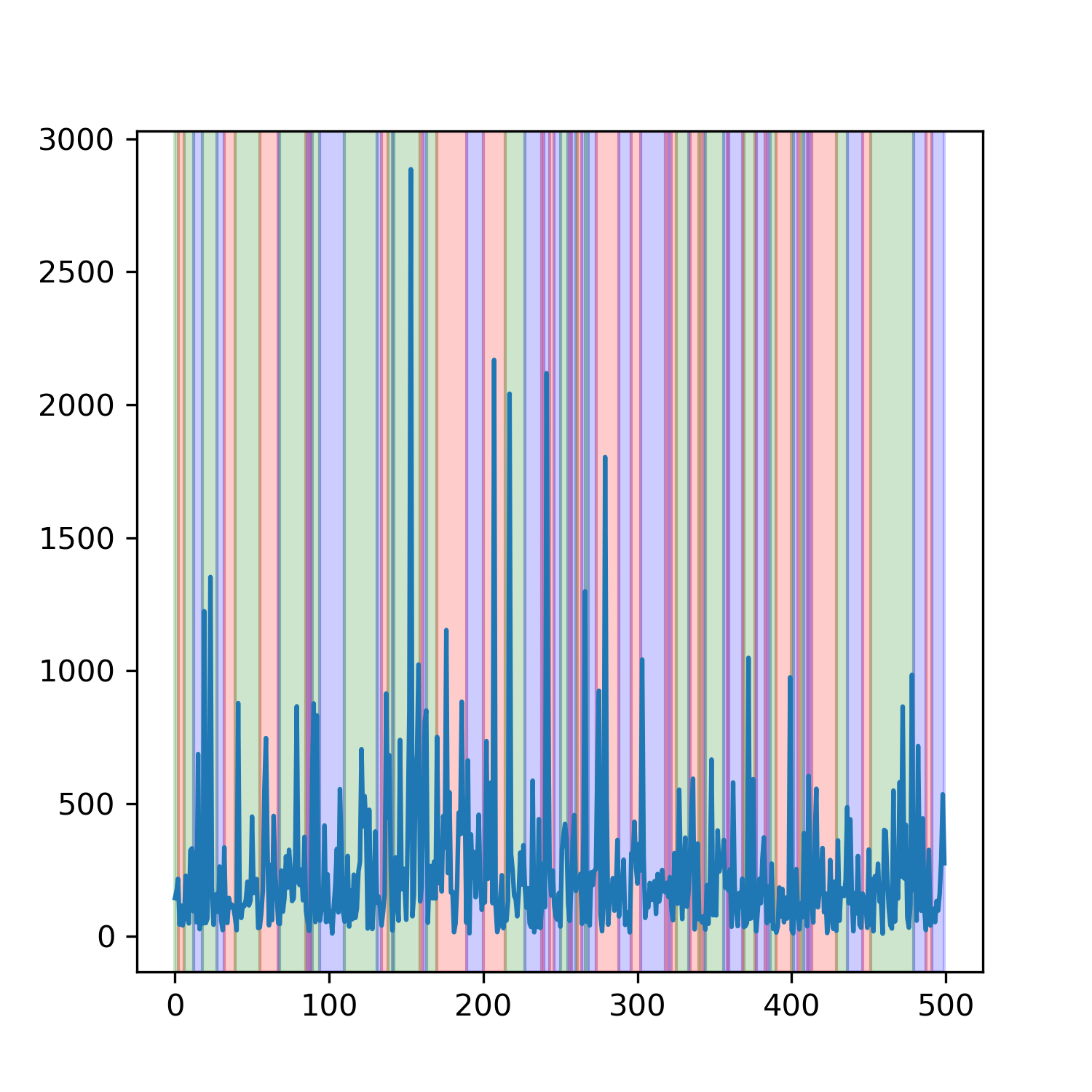}
		\caption{``Earthquake'' process}
	\end{subfigure}
	~
	\begin{subfigure}[b]{0.4\textwidth}
		\centering\includegraphics[width=0.9\columnwidth]{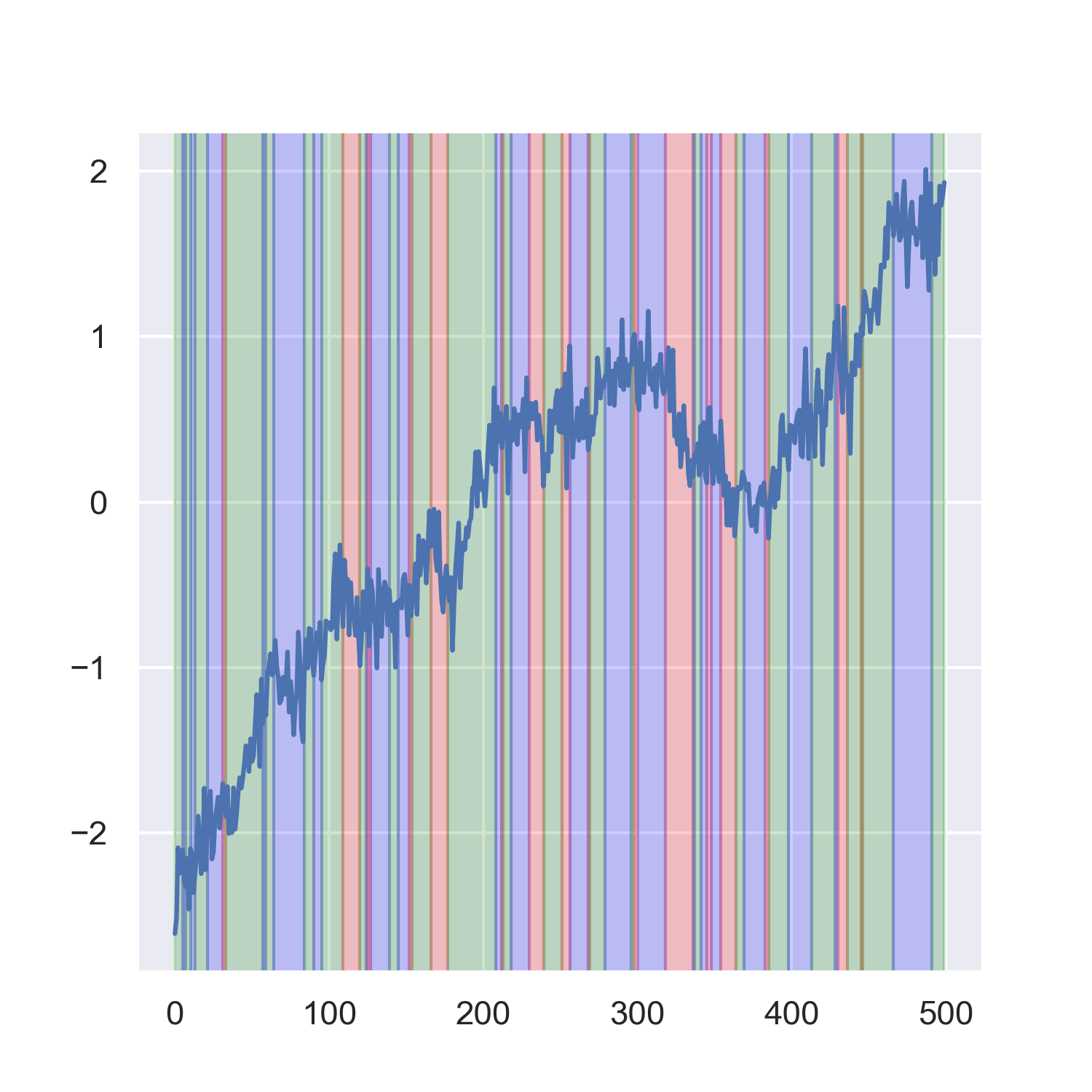}
		\caption{Rapidly changing trend} 
	\end{subfigure}
	\caption{Some trajectories from our model with a three states Markov chain. Up in green, down in red and flat in blue}
	\label{fig:markovian_trajectories}
\end{figure}

\subsection{Training and Validation sets}\label{subsection:trainingandvalidation}
Training sets are made of 1000 time series containing roughly 1000 data points, randomly drawn:
\begin{itemize}
	\item[-] from either one of the three previous dynamics (see section \ref{subsection:tsdynamics})
	\item[-] or from all of the previous dynamics. This will be named mixed dynamic in the following
\end{itemize}
Model selection is made on validation sets composed of 300 time series: 100 samples from each of the three dynamics described in section \ref{subsection:tsdynamics}. Each sample has between 500 and 1000 points depending on the dynamics and the draw. Figure \ref{fig:show_validation_samples} shows random samples from the validation set.
This validation set offers a rich set of scenarios and can be used to assess the ability of an estimator to detect trends.
Hyper-parameters are chosen using a separate test set which is a new random draw of the training set.

\begin{figure}[h]
	\centering
	\includegraphics[width=0.8\textwidth]{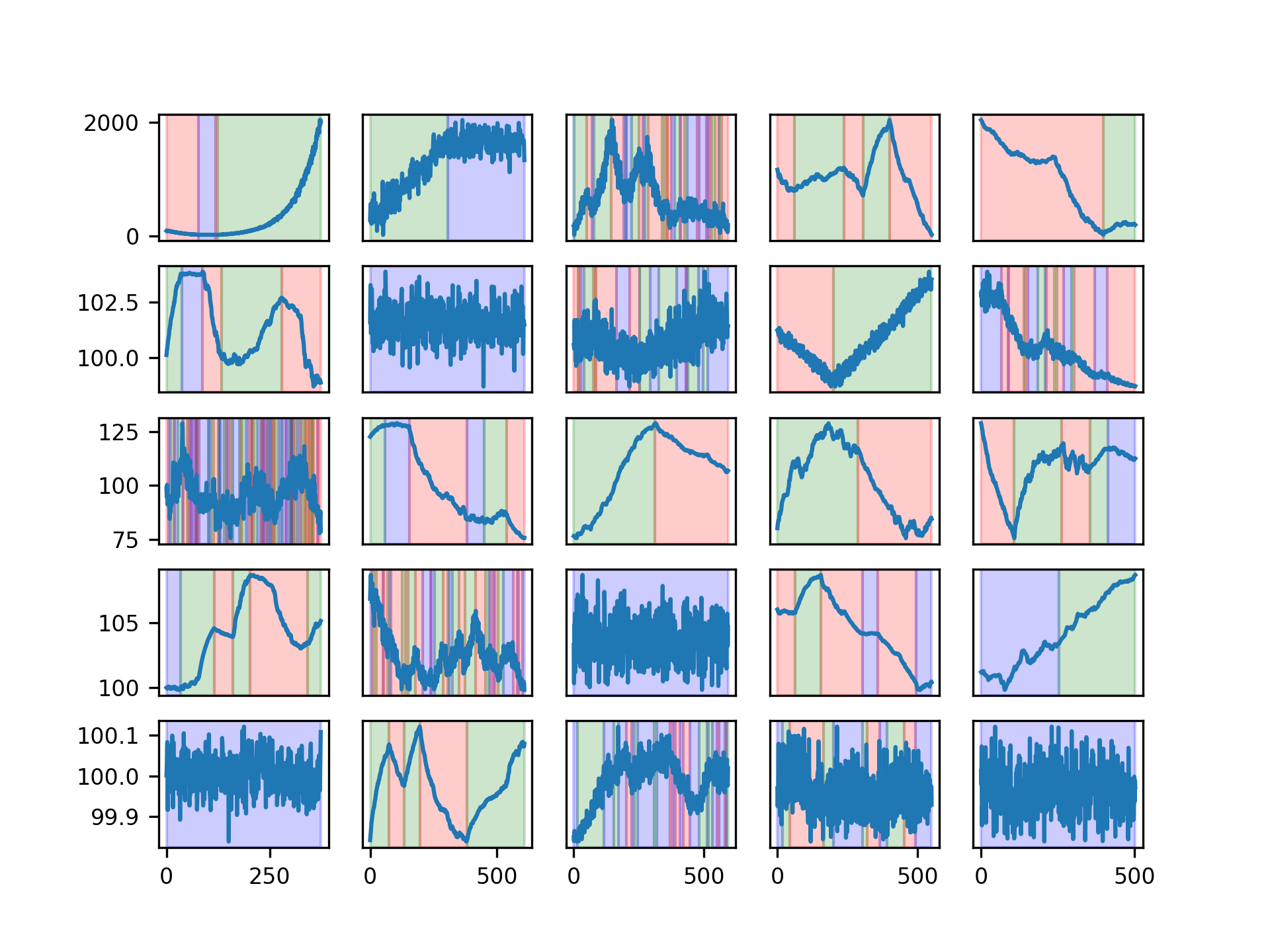}
	\caption{Some samples of a validation set}
	\label{fig:show_validation_samples}
\end{figure}
\clearpage
\subsection{From empirical data to stylised time series dynamics}\label{subsection:tsempiricaltostylized}
One important question arising from the chosen approach is the relevance of the simulated data. The dynamics can show behaviours that, even if not designed to simulate market dynamics, can be relatively similar to actual asset prices. As an example on figure \ref{fig:stylized_trajectories_vs_reality} we plot real assets daily time-series versus a random sample from our three dynamics.
\begin{figure}[h]
	\centering
	\begin{subfigure}[b]{0.4\columnwidth}
		\centering\includegraphics[width=0.9\columnwidth]{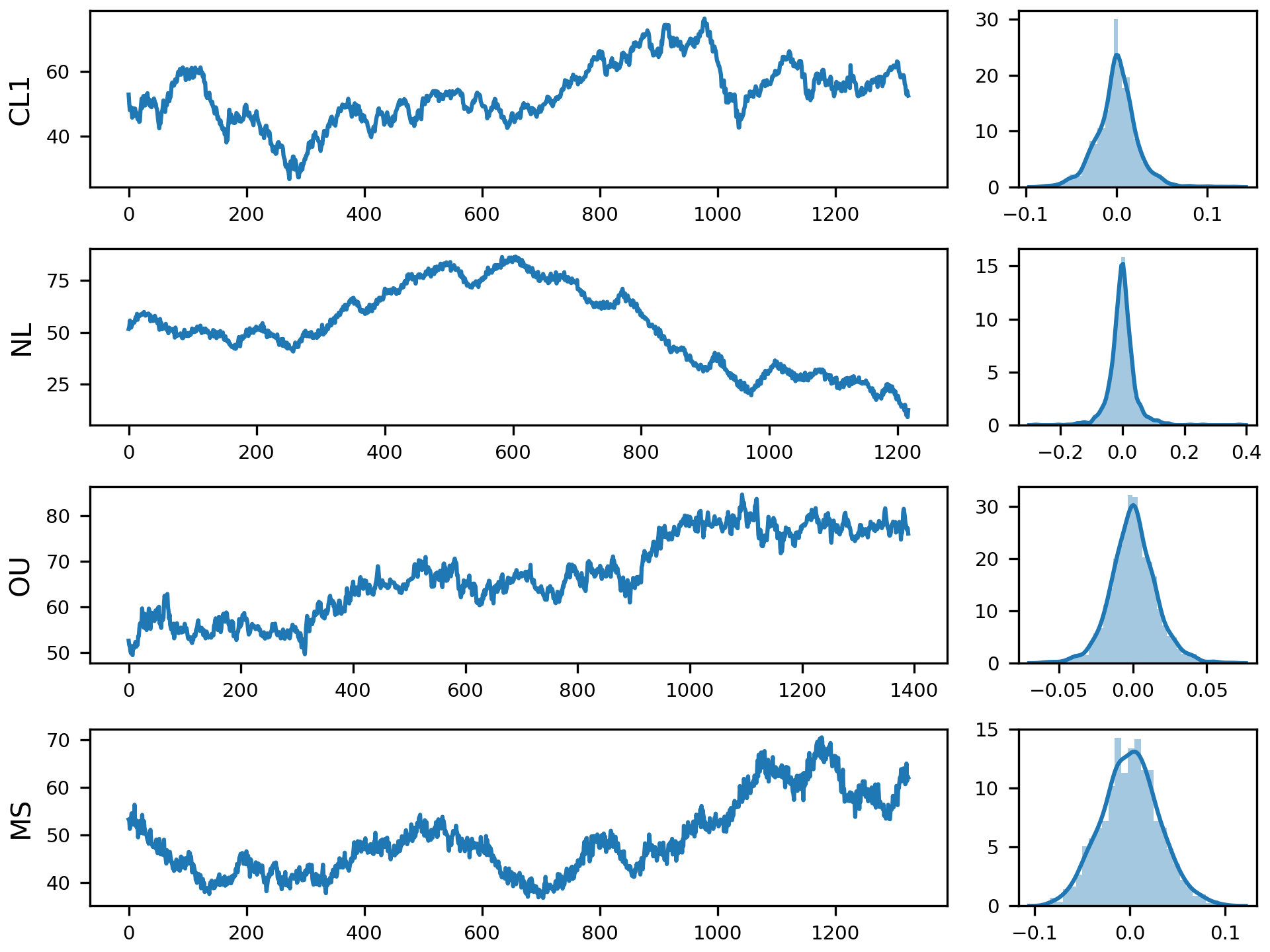}
		\caption{Oil future contract}
	\end{subfigure}
	~
	\begin{subfigure}[b]{0.4\columnwidth}
		\centering\includegraphics[width=0.9\columnwidth]{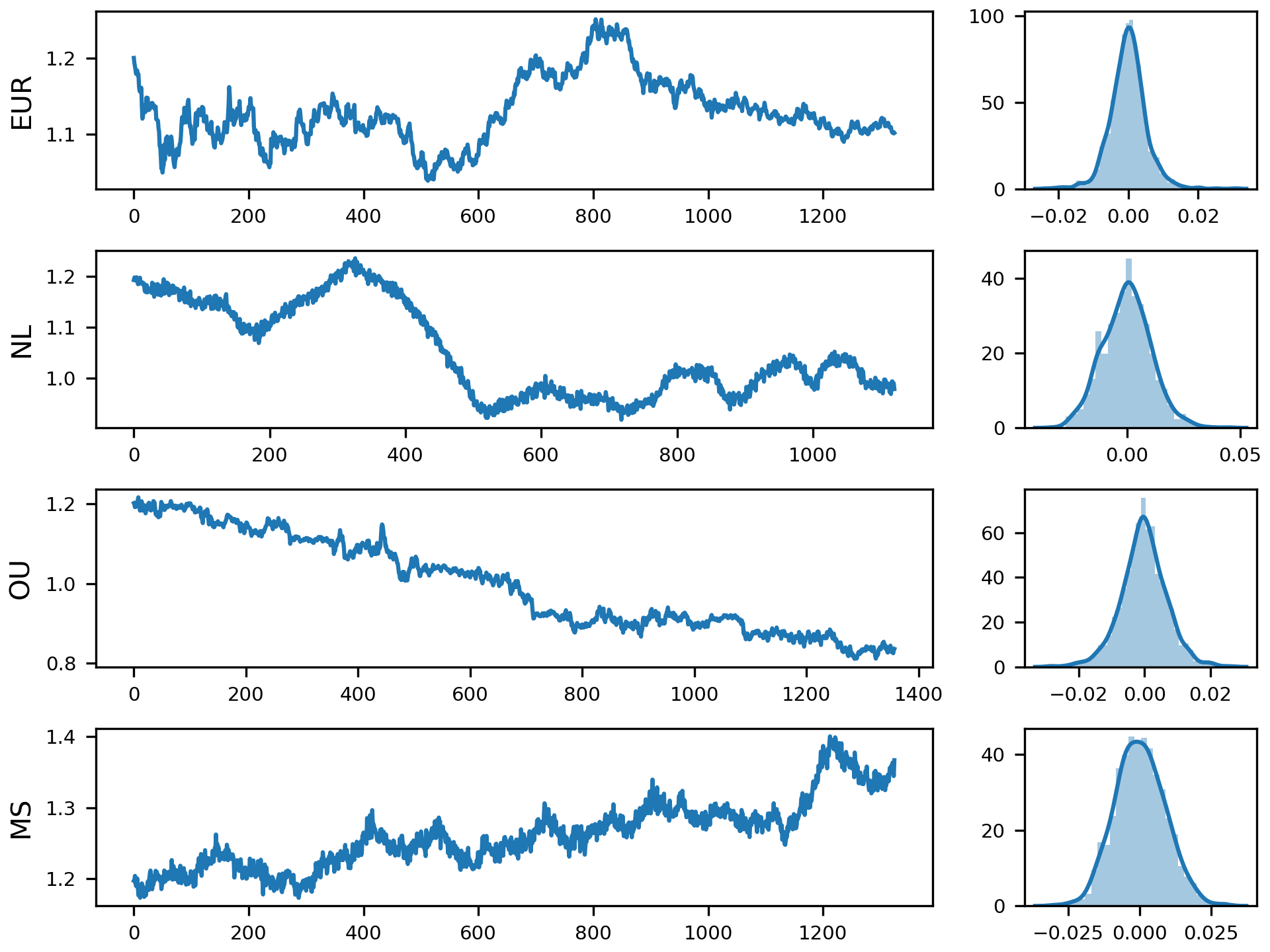}
		\caption{EUR-USD exchange rate}
	\end{subfigure}
	\\
	\begin{subfigure}[b]{0.4\columnwidth}
		\centering\includegraphics[width=0.9\columnwidth]{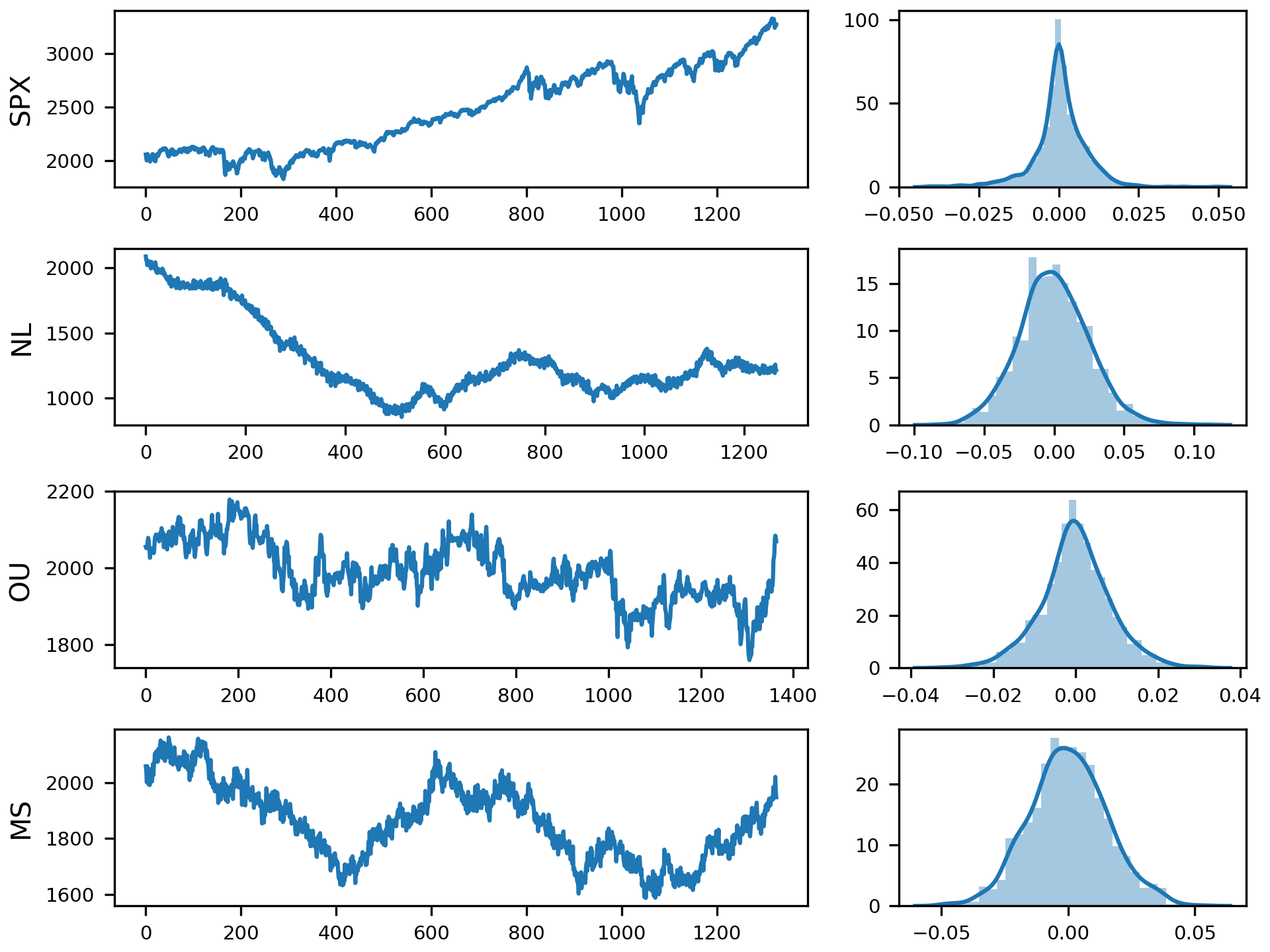}
		\caption{\SP index}
	\end{subfigure}
	~
	\begin{subfigure}[b]{0.4\columnwidth}
		\centering\includegraphics[width=0.9\columnwidth]{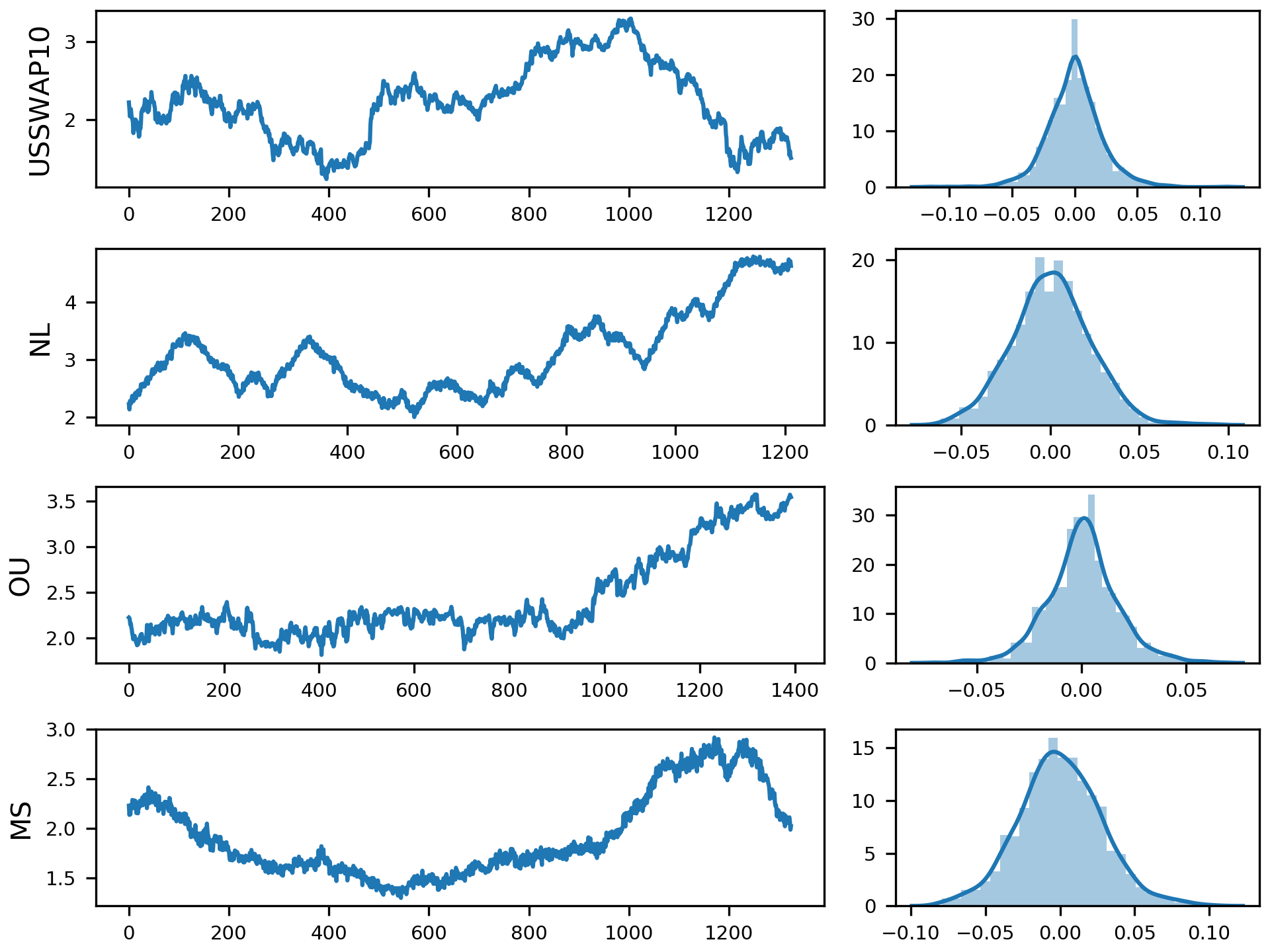}
		\caption{USD 10years swap rate} 
	\end{subfigure}
	\caption{Real assets versus various samples of simulated dynamics}
	\label{fig:stylized_trajectories_vs_reality}
\end{figure}
We see that the trajectories can be visually similar but that the distribution of daily returns may differ greatly. We must bear in mind that our aim is not to simulate market data but to detect trend defined as the sign of the drift term. We think that our dynamics are good enough to simulate this property of real time-series.
One general method to get simulated dynamics close to empirical market data is the following :
\begin{enumerate}
	\item Chose a dynamic
	\item Compute the distribution of returns of the market time series of interest
	\item Sample time-series of the dynamic and compute the distributions returns
	\item Compute the average distance between the sampled distributions and the empirical one\footnote{we used Wasserstein distance for our experiments}
	\item Minimize this function over the dynamic parameters using black-box Bayesian optimization
\end{enumerate}

%% file: tex/section_rnn.tex
\section{Using Recurrent Neural Networks to detect trends}
We motivate here the use of Recurrent Neural Networks (RNN) for our classification problem. Drawing from simple intuition, we provably show their benefits in a simple case.
\subsection{Motivation : moving averages filtering and its extension as RNN}\label{subsection:motivation}
One of the most common way to detect trends is to adopt a filtering approach, comparing smoothed versions of the initial process. For example, we could aggregate several moving averages like: 
\begin{equation}
h^\alpha_t = \alpha h^\alpha_{t-1} + (1-\alpha) Y_t
\label{eqn:ma}
\end{equation}
with various values of $\alpha\in[0, 1]$. Determining the optimal $\alpha$ might be difficult if we want to build an estimator adapted to various dynamics. To circumvent this difficulty, we can aggregate the values for different $\alpha$ as the components of vectors $h_t = (h^{\alpha_1}_t,\,\ldots\, ,h^{\alpha_m}_t)$ through time\footnote{$h_t$ is going to be the hidden state of our RNN}.\\
For example, we might want to consider $h_t = (h^{\alpha_1=0.1}_t, h^{\alpha_2=0.5}_t, h^{\alpha_2=0.9}_t) \in\R^3$ concatenation of a fast, medium and slow moving averages. We might compare:
\begin{itemize}
	\item[-] the slow and the fast moving averages by looking at the sign of
	$$h^{0.9}_t - h^{0.1}_t = \begin{bmatrix}-1\\0\\1\end{bmatrix}^\intercal\cdot \begin{bmatrix}h^{0.1}_t\\h^{0.5}_t\\h^{0.9}_t\end{bmatrix}$$
	\item[-] or maybe the slow versus an average of the medium and slow with the sign of
	$$h^{0.9}_t - \frac{1}{2}(h^{0.1}_t + h^{0.5}_t)= \begin{bmatrix}-0.5\\-0.5\\1\end{bmatrix}^\intercal\cdot \begin{bmatrix}h^{0.1}_t\\h^{0.5}_t\\h^{0.9}_t\end{bmatrix}$$
	\item[-] or whatever weighted combination we fancy with the sign of
	$$0.23 \, h^{0.9}_t + 1.5\,h^{0.5}_t - 0.96\,h^{0.1}_t = \begin{bmatrix}-0.96\\1.5\\0.23\end{bmatrix}^\intercal\cdot \begin{bmatrix}h^{0.1}_t\\h^{0.5}_t\\h^{0.9}_t\end{bmatrix}$$
\end{itemize}
Generally speaking, we look at the signs of components of the vector $W \cdot h_t$ where $W$ is a given\footnote{or more probably learnt} weight matrix. The rows of $W$ define hyperplanes. The half-spaces determined by $W$ are given by the signs of the components of $W \cdot h_t$. Detecting a trend is simply trying to locate $h_t$ with regards to convex polytopes determined by these half-spaces.\\
Generalizing equation \eqref{eqn:ma} to upper dimensions, we have:
$$
h_t = W_{hh}\, h_{t-1} + Y_t \, w_{ih}
$$
where $W_{hh} \in \mathcal{M}^{+}_{m}(\R)$ is a positive matrix and $w_{ih} \in \R^m_{+}$ a positive vector such that
$$
\forall i \in [0, m]\,(w_{ih})_{i} + \sum_{j=1}^{m} (W_{hh})_{(i, j)} = 1
$$
The trend is determined by $\sgn(W\cdot h_t)$ but we could use any other activation function $f$ instead of the sign function.

These equations are exactly equal to the update equation of a RNN composed of
\begin{itemize}
	\item a vanilla RNN
	\begin{itemize}
		\item with the identity as activation function
		\item with one hidden layer
		\item with convex constraints on the weight matrix $\begin{bmatrix}W_{hh}, w_{ih}\end{bmatrix}$\footnote{which is therefore a stochastic matrix}
	\end{itemize}
	\item with a simple linear layer and activation function $f=\sgn$
\end{itemize}
Such a RNN will be called a ``convex net'' in the following. This shows that RNNs can be considered as generalizations of some basic moving average comparisons.
As a working example, we consider the case of the Noisy Line Process $Y_t = Y_0 + \mu t+ \epsilon_t$ where $\epsilon_t$ are independent noise random variables $\E(\epsilon_t) = 0$.\\For a net with constrained weights it can be shown (see annex \ref{annex:convexcell} for details):
\begin{itemize}
	\item[-] without trend, $\mu=0$, then $\{h_t\}$ becomes centered around a variable of finite variance 
	\item[-] with trend, $\mu \neq 0$ then $\{h_t\}$ diverges
\end{itemize}
If we now introduce a hyperbolic tangent activation function instead of identity:
\begin{itemize}
	\item[-] if $\mu=0$, near zero the cell is in the linear part and we should expect the state to stay bounded around the origin
	\item[-] if the trend $\mu\neq 0$ then the state should go towards $\sgn(\mu) \times \infty$ i.e. to navigate near the faces of the $]0, 1[^n$ hypercube
\end{itemize}
For a practical illustration see annex \ref{annex:hiddenstateplot}.

\subsection{Overview of RNNs and data}
\subsubsection{Standards Recurrent Neural Nets}
In subsection \ref{subsection:framework}, we turned the trend estimation problem into a sequence to sequence classification task, for which RNNs can be used.
We consider three standard structures:
\begin{itemize}
	\item[-] Vanilla RNN as defined in \cite{Elman1990}
	\item[-] LSTM as introduced in \cite{Hochreiter1997}
	\item[-] GRU as introduced in \cite{Cho2014}
\end{itemize}
RNNs contain cycles: hidden state cell can depend on the entire past input sequence.
We refer to \cite{Graves2012} for details. These three standard RNNs have different structures but they share similar update equations like:
$$
g_{t} = f \circ 
\begin{bmatrix}
W^x_{1} & W^h_{1} \\
\vdots & \vdots \\
W^x_{n} & W^h_{n} \\
\end{bmatrix}
\cdot
\begin{bmatrix}
Y_t\\
h_{t-1}
\end{bmatrix}
$$
where
\begin{itemize}
  \item[-] $g_t$ is a vector representing some internal cells at $t$
  \item[-] $f=(f_1,\ldots,f_n)$ is an block-wise activation function
  \item[-] $Y_t$ is the input at time $t$
  \item[-] $h_{t-1}$ is the state at time $t-1$
  \item[-] $W^h_{i}$ are matrices and  $W^x_{i}$ vectors 
\end{itemize}
$\circ$ is a elementwise application operator \footnote{e.g. \begin{equation*}\begin{bmatrix}f_1\\f_2\end{bmatrix}\circ\begin{bmatrix}x_1\\x_2\end{bmatrix}=\begin{bmatrix}f_1(x_1)\\f_2(x_2)\end{bmatrix}\end{equation*}} and $\cdot$ the matrix product.\\Depending on the RNN, $h_t$ is a combination of blocks of $g_t$ and possibly $g_{t-1}$.\\Essentially, $h_{t}= F(Y_t, h_{t-1})$ where $F$ is a possibly complex mapping from the previous state and actual input values to the new state. We refer the reader to \cite{Elman1990}, \cite{Hochreiter1997} and \cite{Cho2014} for more details.

\subsubsection{Training RNNs}\label{subsubsection:trainingandvalidation}
For training and validation, we use simulated time series according to section \ref{subsection:trainingandvalidation}. Our aim is to give a precise empirical comparison of these three structures taking into account the possible influence of the training dynamic. We train triplets of the form:
\begin{itemize}
	\item[-] a RNN chosen among Vanilla, LSTM or GRU
	\item[-] some meta-parameters like the number of recurrent layers, the dimension of hidden layer(s), dropout (see \cite{Srivastava2014} for definition)\ldots
	\item[-] a time series dynamic chosen among Noisy Line Process, Piecewise \OU, Markovian Switch or a mixed dynamic
\end{itemize}
Each of these triplets is trained and validated against the training and validation sets described in subsection \ref{subsection:trainingandvalidation}. This gives us more than 400 triplets to train and validate. Roughly 100 triplets do hit convergence issues in the training period and are excluded from the validation phase.
Some parameters details can be found in annex \ref{annex:rnntrainingdetails}.
Also, to get more robust results, we did a complete training using two different gradient step optimizations:
\begin{itemize}
	\item[-] Adam (see \cite{Kingma2014} for details) as it is commonly used and has some theoretical convergence properties to a stationary point (see \cite{Barakat2018} for details)
	\item[-] RMSprop algorithm (see \cite{Hinton2012} for details)
\end{itemize}

\subsection{Empirical findings}\label{subsection:empiricalfindings}
We train our triplets as described in subsection \ref{subsubsection:trainingandvalidation} for both Adam and RMSprop and validate each triplet on our 300 validation samples (see section \ref{subsection:trainingandvalidation}). The loss is a binary loss on the labels.\\
Table \ref{tab:coeffsimpleols} shows the coefficients of the linear regression of loss against binary variables indicating the training dynamic, the net type, the optimization type and the validation dynamic. Each feature is translated into binary on/off variables with one less modality. The missing modality is on if all others are set to zero. A positive coefficient means that the highlighted feature increases the average loss of the sample, and conversely, a negative coefficient decreases the average loss. Full details can be found in annex \ref{annex:rnnempiricalfindings}.~\\
\begin{table}[h]
	\centering
	\begin{tabular}{|l|c|}
		\hline
		Feature[Modality]&Coefficient\\
		\hline
		Intercept&0.48\\
		Training dynamic[Markovian Switch]&$\approx 0$\\
		Training dynamic[\OU]&0.029\\
		Training dynamic[Noisy Line]&$\approx 0$\\
		Net Type[LSTM]&0.037\\
		Net Type[Vanilla]&0.17\\
		Optimization[RMSP]&0.0234\\
		Validation dynamic[\OU]&-0.1\\
		Validation dynamic[Noisy Line]&-0.036\\
		\hline
	\end{tabular}
	\caption{Ordinary least squares (OLS) model of the loss onto the various features. Left hand column is the feature column with the specified modality in brackets. Positive coefficient means that the presence of the modality in brackets is detrimental to performance}
	\label{tab:coeffsimpleols}
\end{table}
\linebreak
From figure \ref{fig:plot_optimize_net_training}:
\begin{itemize}
	\item[-] training on \OU\space dynamic seems to worsen performance
	\item[-] GRU seems to be the best net type and Vanilla not a great choice
	\item[-] the optimization algorithm RMSProp has a negative impact on performance. Adam leads to better results
	\item[-] the validation loss for Markovian Switch is higher than the two other dynamics
\end{itemize}
\begin{figure}[h]
	\centering
	\includegraphics[width=0.8\columnwidth]{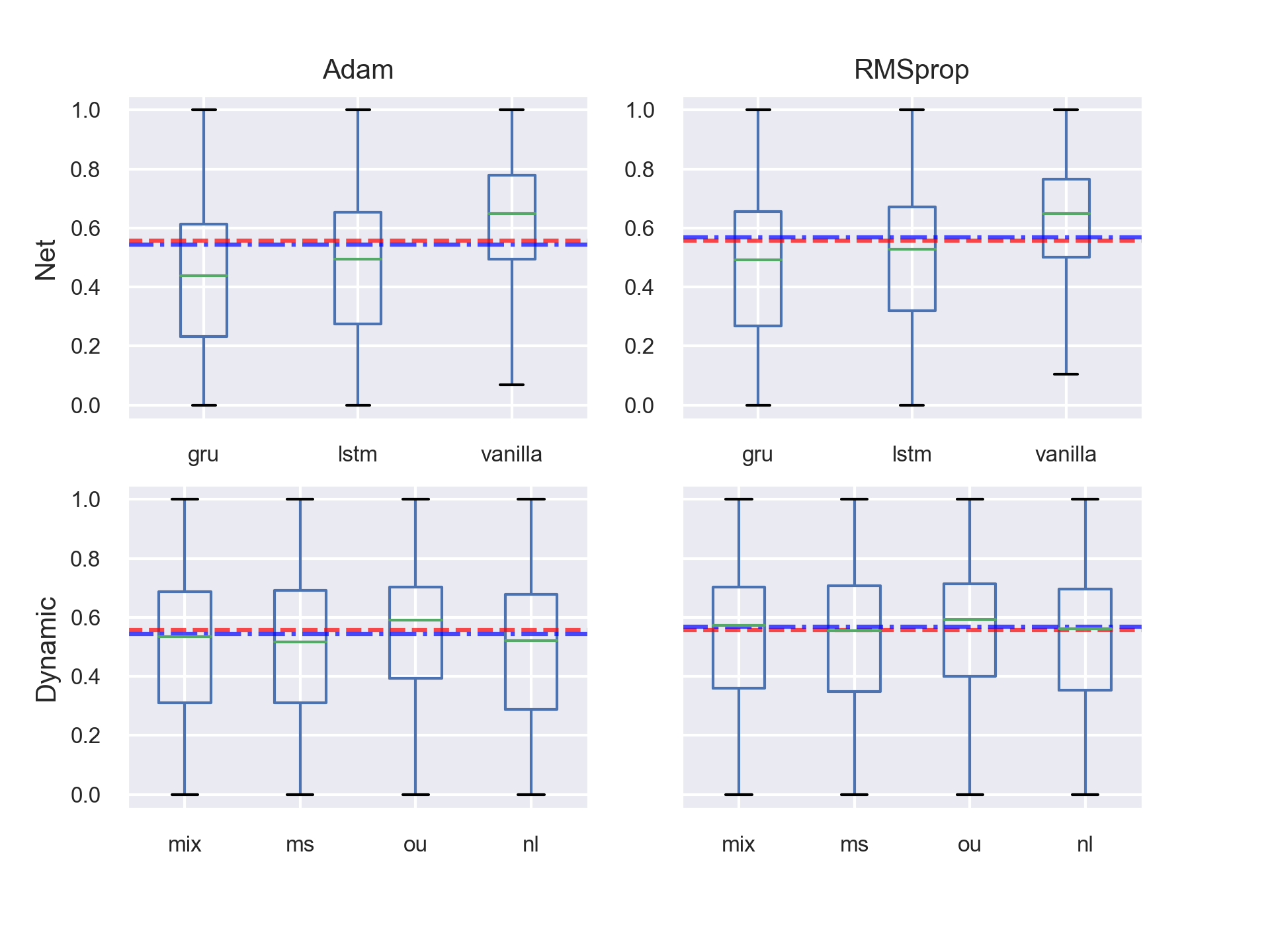}
	\caption{Box-plotting losses by optimization, net type and training dynamic. In dashed red the overall median loss, in dash-dotted blue the overall loss for a given optimization type. Dynamic of the training data is nl for Noisy Line, ou for Piecewise \OU , ms for Markovian Switch and mix for the mixed dynamic}
	\label{fig:plot_optimize_net_training}
\end{figure}
\pagebreak
\vspace{0.5em}
\emph{Training dynamic} has an impact on validation performance. Choosing two dynamics e.g. Noisy Line versus Piecewise \OU, we select data from those only and bootstrap. For each bootstrapping iteration, we compute the difference between the medians of losses of one dynamic versus the other. The result can be seen on table \ref{tab:bootstraptraining}.
Even if all intervals contain zero, and no robust conclusion can be drawn, the median loss seems lower when training using the Noisy Line or Markovian Switch dynamics.
\begin{table}[h]
	\centering
	\begin{tabular}{|l|c|@{\hskip 0.1in}p{2cm}|@{\hskip 0.1in}p{2cm}|}
		\hline
		type 1 - type 2&Median loss difference&\multicolumn{2}{c|}{1\% confidence interval}\\
		\hline
		nl - ou&-0.04&\cellcolor{red!25}-0.19& \cellcolor{blue!25}0.10\\
		nl - ms&0.01& \cellcolor{red!25}-0.15& \cellcolor{blue!25}0.17\\
		nl - mix&-0.009&  \cellcolor{red!25}-0.17&  \cellcolor{blue!25}0.15\\
		ou - ms&0.05&  \cellcolor{red!25}-0.10&  \cellcolor{blue!25}0.21\\
		ou - mix&0.04&  \cellcolor{red!25}-0.12&  \cellcolor{blue!25}0.20\\
		ms - mix&-0.02&  \cellcolor{red!25}-0.20&  \cellcolor{blue!25}0.16\\
		\hline
	\end{tabular}
	\caption{Difference of median loss for training type 1 -  median loss for training type 2 using bootstrapping percentile confidence interval. In red, negative values, blue, positive values, in confidence interval columns}
	\label{tab:bootstraptraining}
\end{table}
\vspace{0.5em}
\newline\emph{Net structure} are compared using the same bootstrapping procedure in table \ref{tab:bootstrapnet}. Vanilla RNN is consistently worse than LSTM and GRU at 99\% confidence level. As a result, in the following, we will ignore triplets with Vanilla RNN. Vanilla RNN is barely better than a dummy estimator having $\frac{1}{3}$ chance of correctly predicting the trend (see annex \ref{annex:vanillaisdummy}).
\begin{table}[h]
	\centering
	\begin{tabular}{|l|c|@{\hskip 0.1in}p{2cm}|@{\hskip 0.1in}p{2cm}|}
		\hline
		net 1 - net 2&Median loss difference&\multicolumn{2}{c|}{1\% confidence interval}\\
		\hline
		vanilla - lstm&0.14& \cellcolor{yellow!25}-0.005& \cellcolor{yellow!25}0.28\\
		vanilla - gru&0.18& \cellcolor{yellow!25}0.04& \cellcolor{yellow!25}0.32\\
		lstm - gru&0.05&-0.15&0.25\\
		\hline
	\end{tabular}
	\caption{Difference of median loss for net structure 1 -  median loss for net structure 2 using bootstrap percentile confidence interval. Highlighted in yellow the underperformance of Vanilla RNN}
	\label{tab:bootstrapnet}
\end{table}
\vspace{0.5em}
\newline\emph{Optimizer} impact: results seem to indicate a slightly better performance of Adam versus RMSprop\footnote{$\text{median loss}_{\text{Adam}} - \text{median loss}_{\text{RMSP}} \approx -0.04$ with a confidence interval equal to $[-0.27, 0.18]$}.
\vspace{0.5em}
\newline\emph{Net structure and training dynamic interaction}: using only the triplets where net structure is either GRU or LSTM, we run the same bootstrapping procedure for each datasets on the training dynamic. The results are given in table \ref{tab:trainnetinteraction}. All the intervals contain 0 and it is difficult to find a combination which does significantly better than the others.

\begin{table}[h]
	\centering
	\setlength{\tabcolsep}{3pt}
	\begin{subfigure}[b]{0.9\columnwidth}
		\centering
		\begin{tabular}{|l|c|p{2cm}|p{2cm}|}
			\hline
			type 1 - type 2&Median loss difference&\multicolumn{2}{|c|}{1\% confidence interval}\\
			\hline
				nl - ou&-0.05&-0.25&0.15\\
				nl - ms&0.002&-0.18&0.19\\
				nl - mix&-0.002&-0.19&0.19\\
				ou - ms&0.05&-0.10&0.20\\
				ou - mix&0.05&-0.12&0.22\\
				ms - mix&-0.005&-0.18&0.17\\
			\hline
		\end{tabular}
		\caption{Training bootstrap for LSTM only}
		\label{tab:trainingbootstraplstm}
	\end{subfigure}
	\\
	\begin{subfigure}[b]{0.9\columnwidth}
		\centering
		\begin{tabular}{|l|c|p{2cm}|p{2cm}|}
			\hline
			type 1 - type 2&Median loss difference&\multicolumn{2}{|c|}{1\% confidence interval}\\
			\hline
			nl - ou&-0.025&-0.21&0.17 \\
			nl - ms&0.05&-0.13&0.24 \\
			nl - mix&0.06&-0.13&0.26 \\
			ou - ms&0.08&-0.11&0.26 \\
			ou - mix&0.09&-0.08&0.25 \\
			ms - mix&0.008&-0.18&0.20 \\
			\hline
		\end{tabular}
		\caption{Training bootstrap for GRU only}
		\label{tab:trainingbootstrapgru}
	\end{subfigure}
	\caption{Interaction between the net structure GRU or LSTM and the training type Noisy Line (nl), Piecewise \OU\space (ou) or Markovian Switch (ms). The loss difference is the loss of the first element of the pair minus the loss of the second}
	\label{tab:trainnetinteraction}
\end{table}

\clearpage
\subsection{RNN baseline selection}\label{subsection:rnnbaseline}
We would like to choose a RNN estimator having a good overall performance on validation data. As we have seen, it is difficult to choose a particular training type or net structure (GRU or LSTM) as being significantly better. 
A way to build a baseline would be for example to pool the estimated probabilities of the best trained estimators. The pooling function here is a simple average of each estimated probabilities from the selected estimators\footnote{see \cite{Adamčík2014} for a justification}. And this, indeed, gives good results on validation data as can be seen in table \ref{tab:stackednetsperf}. We note little difference in performance when pooling more than five estimators.
\begin{table}[h]
		\centering
		\scalebox{0.9}{
		\begin{tabular}{|l|c|c|c|c|}
			\hline
			Validation dynamic type & Median loss & First quartile & Third quartile & IQR \\
			\hline
			Mixed& 0.22 & 0.11 & 0.39 & 0.28 \\
			\OU&  0.21 &  0.14 &  0.31 &  0.17 \\
			Markovian Switch&  0.37 & 0.21 &  0.52 &  0.31 \\
			Noisy Line& 0.11 & 0.05 &  0.23 &   0.18 \\
			\hline
		\end{tabular}
		}
		\caption{Loss and Interquartile Range (IQR) of loss for the pooled net of 5 best RNN estimators}
		\label{tab:stackednetsperf}
\end{table}

Yet, choosing such an estimator would give RNNs an advantage compared to other estimators. To be as fair as possible and favour simplicity over performance we choose to optimize hyper-parameters for a GRU network trained on the Piecewise Noisy Line dynamic using Adam optimization. Some details of the RNN baseline can be found in table \ref{tab:rnn_baseline}.\\

It is interesting to note that adding training epochs\footnote{reasonably from 100 epochs to 200. Going towards 1000 epochs for example gives a marginal improvement in performance but with increasing variance hinting for overfitting} seems to slightly increase the median error on the test set but gives a noticeable decrease of the \IQR\space by a factor near 25\%.\\
	\begin{table}[h]
		\centering
		\begin{tabular}{|l|c|}
			\hline
			Net structure type& GRU\\
			Dropout& 0.2\\
			Number of hidden recurrent layers & 2\\
			Dimension of hidden recurrent layers & 20\\
			Learning rate& 0.005\\
			Number of epochs& 200\\
			\hline
			Training type& Noisy Line\\
			Max noise level& 0.07\\
			Max line slope& 1.4\\
			\hline
		\end{tabular}
		\caption{Parameters of RNN baseline}
		\label{tab:rnn_baseline}
	\end{table}
Running the training with hyper-parameters not too far from the ones obtained by optimization gives fairly similar results. The comparison of the RNN baseline versus the pooled estimator is given in table \ref{tab:compare_pooled_net_baseline_medians} and figure \ref{fig:compare_pooled_net_baseline} for the loss distributions. Even if our RNN baseline is not the best it still offers good performance.
\begin{figure}[h]
	\centering
	\includegraphics[scale=0.5]{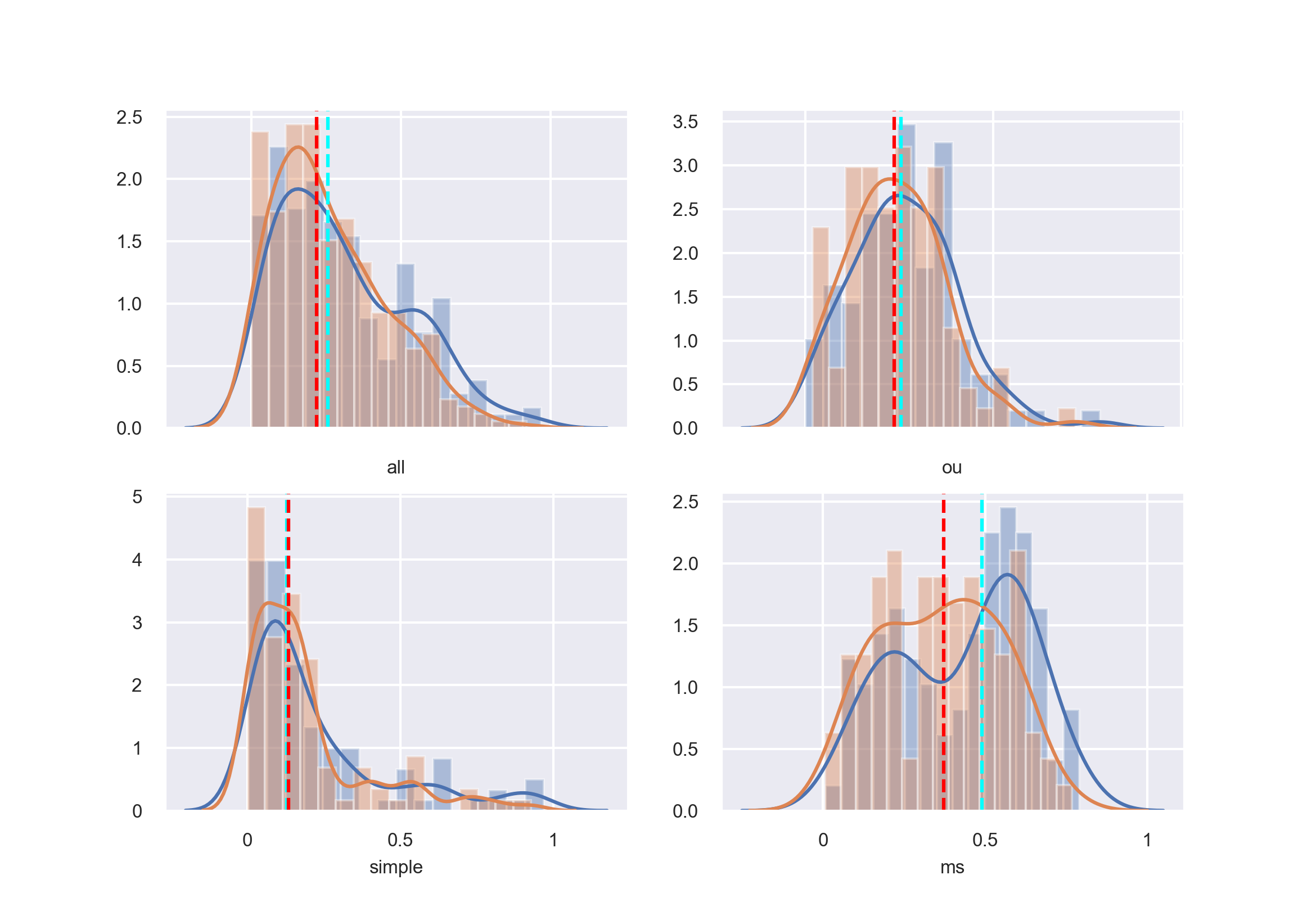}
	\caption{Comparing validation loss distribution for pooled estimator in orange with red median and RNN baseline in blue with cyan median}
	\label{fig:compare_pooled_net_baseline}
\end{figure}

\begin{table}[h]
	\centering
	\begin{tabular}{|l|c|c|}
		\hline
		Dynamic & RNN & Pooled estimator \\
		\hline
		All & 0.25 & 0.22 \\
		\OU & 0.25 & 0.24 \\
		Noisy Line & 0.13 & 0.13 \\
		Markovian Switch & 0.49 & 0.37 \\
		\hline
	\end{tabular}
	\caption{Median losses for RNN baseline or pooled estimator for various dynamics on validation set}
	\label{tab:compare_pooled_net_baseline_medians}
\end{table}
\pagebreak

%% file: tex/section_non_model_based.tex
\section{Non model based estimation}\label{section:nonmodel}
By ``non model based'', we mean estimators which are not based on an explicit modelling of the underlying dynamic. We compare RNN baseline of subsection \ref{subsection:rnnbaseline} against a simple moving average estimator, its generalization (see section \ref{subsection:motivation}) and a Convolutional Neural Network (CNN see \cite{LeCun1998}). Overall, the RNN baseline exhibits much stronger validation performance.

\subsection{Comparison with moving average}\label{subsection:movingaveragecomp}

One of the most intuitive way to detect trend is to compare the speed of two moving averages. We compare our RNN baseline with both the most simple moving average filtering and the convex net generalization approach.
\subsubsection{Simple moving average}
We first compare the RNN baseline with a basic estimator computing two moving averages: a ''s=slow'' one and a ''f=fast'' one
$$
ma^{\text{speed}}_t  = \mu_{\text{speed}} \, ma^{\text{speed}}_t + (1-\mu_{\text{speed}}) \, x_t \quad\text{with speed}\in\text{\{s, f\}}\text{ .}
$$
Given $\epsilon>0$, a no trend threshold, the trend prediction is made by
\begin{align*}
ma^{\text{fast}}_t - ma^{\text{slow}}_t > \epsilon & \Rightarrow \text{up trend}\\
ma^{\text{fast}}_t - ma^{\text{slow}}_t < -\epsilon & \Rightarrow \text{down trend}\\
\text{otherwise} & \Rightarrow \text{no trend}
\end{align*}
Obviously, the parameters $\mu_s\,,\,\mu_f\,,\,\epsilon$ have a big impact on the estimator performance. Using Bayesian optimization we find the parameters shown in table \ref{tab:ma_baseline_parameters}.
\begin{table}[h]
	\centering
	\begin{tabular}{|c|c|}
		\hline
		Parameter & Value \\
		\hline
		$\mu_s$ & 0.95\\
		$\mu_f$ & 0.48\\
		$\epsilon$ & 0.1\\
		\hline
	\end{tabular}
	\caption{Parameters of Moving average baseline}
	\label{tab:ma_baseline_parameters}
\end{table}

On figure \ref{fig:compare_ma_net_baseline} we see the loss distribution of the baseline RNN versus the loss distribution of the moving average estimator for all dynamics.
\begin{figure}[h]
	\centering
	\includegraphics[scale=0.5]{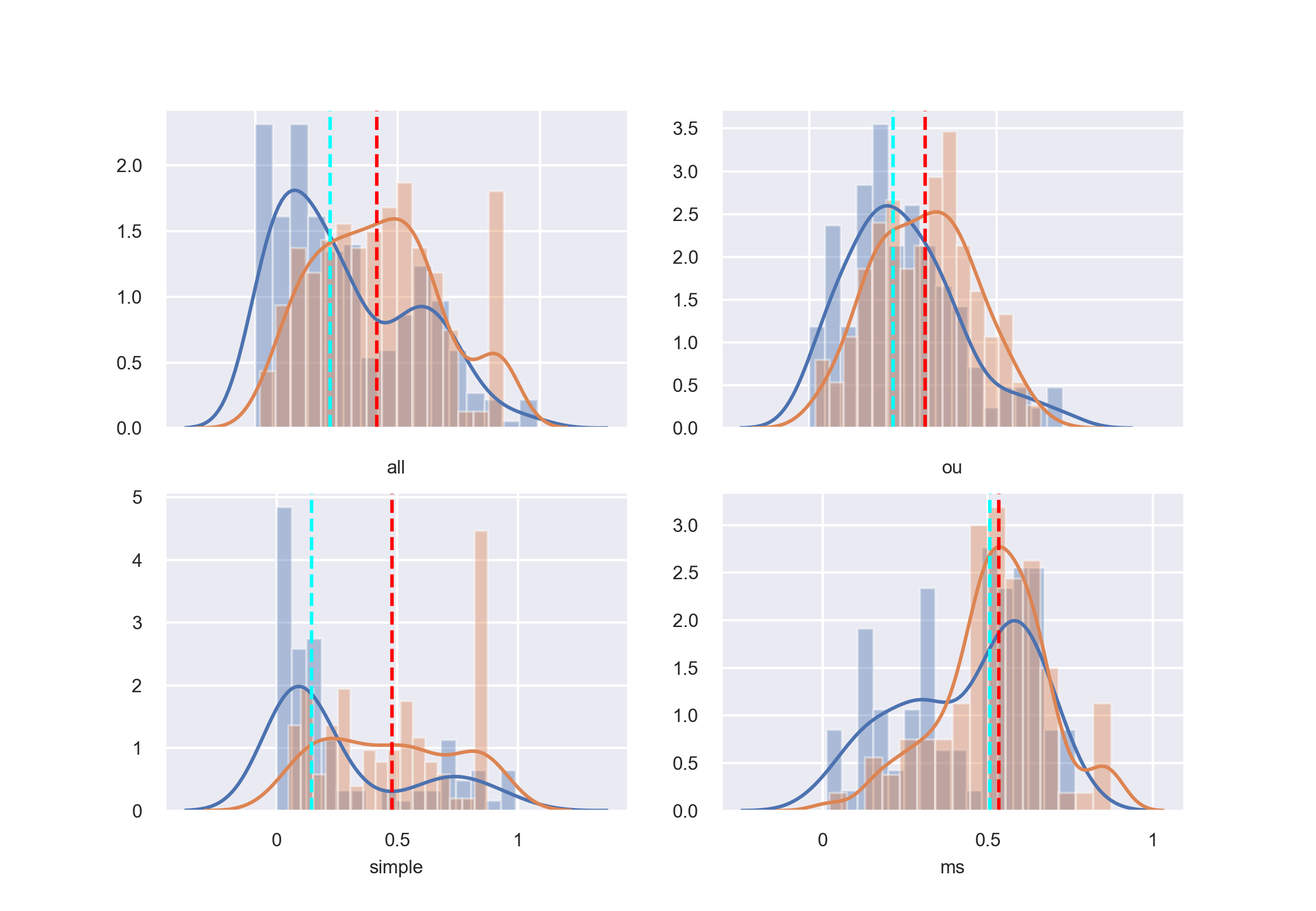}
	\caption{Comparing validation loss distribution for MA estimator in orange with red median and RNN baseline in blue with cyan median}
	\label{fig:compare_ma_net_baseline}
\end{figure}
\pagebreak
On average, the RNN baseline is consistently better than the moving average estimator as seen on table \ref{tab:compare_ma_net_baseline_medians}. The Markovian Switch dynamic is sometimes extremely difficult to apprehend due to highly volatile regime switching. For this dynamic, we see that both estimators are equally bad which is not unexpected given the task difficulty.

\begin{table}[h]
	\centering
	\begin{tabular}{|l|c|c|}
		\hline
		Dynamic & RNN & MA \\
		\hline
		All & 0.26 & 0.43 \\
		\OU & 0.23 & 0.31 \\
		Noisy Line & 0.14 & 0.48 \\
		Markovian Switch & 0.51 & 0.53 \\
		\hline
	\end{tabular}
	\caption{Median loss for RNN or MA estimator for various dynamics on validation set}
	\label{tab:compare_ma_net_baseline_medians}
\end{table}

\subsubsection{Comparison with moving average generalization}
We compare the baseline RNN with the estimator built according to subsection \ref{subsection:motivation}. Basically, this is a Vanilla RNN without any activation function. Also, weights are constrained to be a stochastic matrix. It turns out, a bit surprisingly to us, that the performance is quite poor and way worse than the RNN baseline. Further investigation is needed, but training seems to fail somehow as the trained weights are all very close to zero. As a result, the input plays little role in the prediction and surely can't do much better than a dummy estimator. For reference, basic results are shown in table \ref{tab:compare_cnx_net_baseline_medians}.

\begin{table}[h]
	\centering
	\begin{tabular}{|l|c|c|}
		\hline
		Dynamic & RNN & Generalized moving average \\
		\hline
		All & 0.27 & 0.61 \\
		\OU & 0.26 & 0.61 \\
		Noisy Line & 0.12 & 0.62 \\
		Markovian Switch & 0.47 & 0.61 \\
		\hline
	\end{tabular}
	\caption{Median loss for RNN baseline and convex net estimator for various dynamics on validation set}
	\label{tab:compare_cnx_net_baseline_medians}
\end{table}

\subsection{Comparison with CNN}
One dimensional CNN is sometimes seen as a good tool to analyse time series. We use a standard CNN structure stacking convolutional layer followed by a pooling layer. To keep nets architecture similar in term of parameters, we use two layers of convolution + pooling.
After optimization, we get hyper-parameters shown in table \ref{tab:cnn_baseline_parameters}. Interestingly, both channel and kernel have taken the maximum value in the range we tested\footnote{from 3 to 20}.

\begin{table}[h]
	\centering
	\begin{tabular}{|c|c|}
		\hline
		Parameter & Value \\
		\hline
		Learning rate & 0.004\\
		Channel dimension & 20\\
		Kernel size & 20\\
		\hline
	\end{tabular}
	\caption{Parameters of CNN baseline}
	\label{tab:cnn_baseline_parameters}
\end{table}

Yet, we are unable to find the supposed general efficiency of CNNs in our setup as seen on figure \ref{fig:compare_cnn_net_baseline}.
Actually, CNN performance is barely better than a dummy classifier as seen on table \ref{tab:compare_cnn_net_baseline}.

\begin{figure}[h!]
	\centering\includegraphics[scale=0.5]{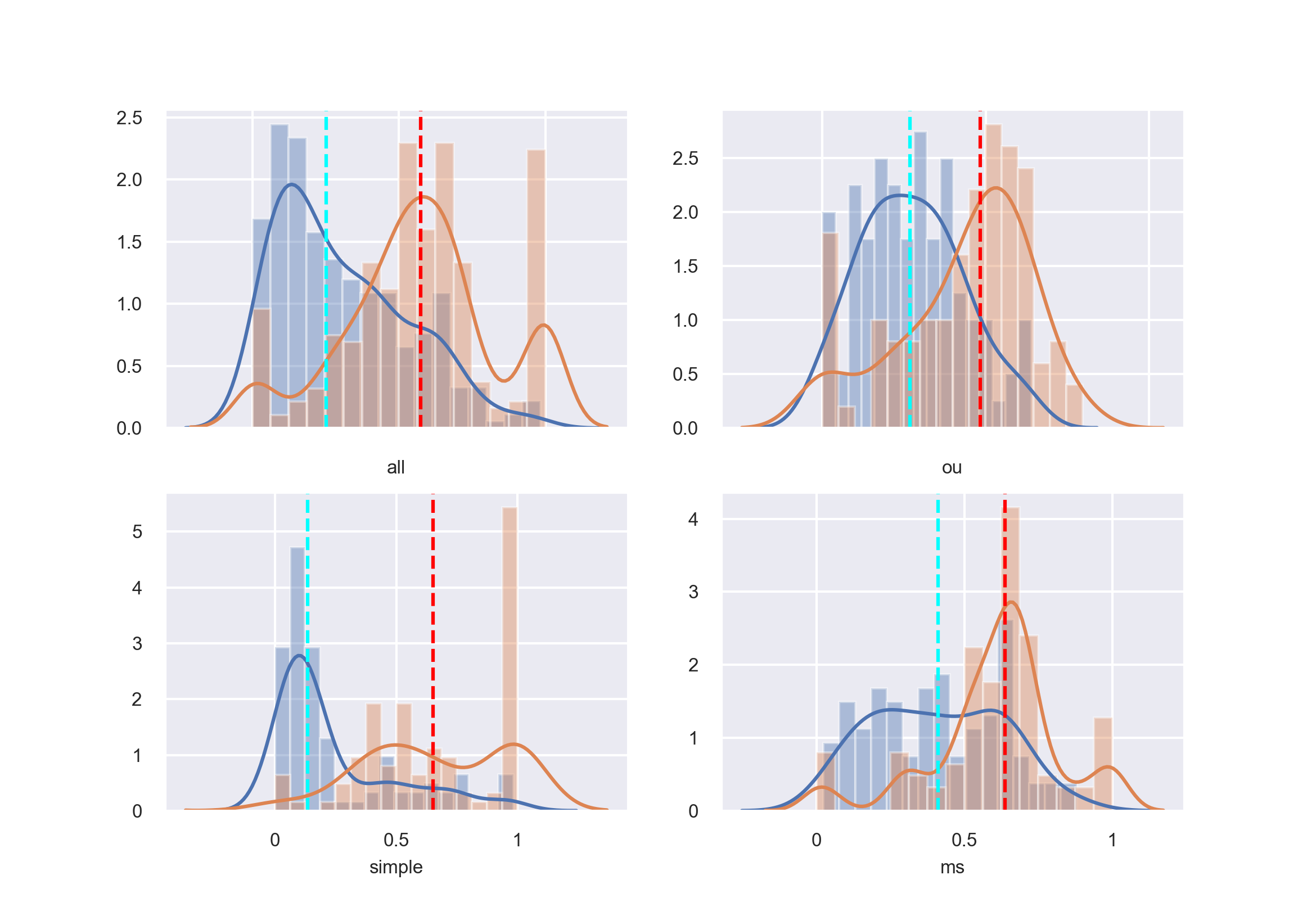}
	\caption{Comparing validation loss distribution for CNN estimator in orange with red median and RNN baseline in blue with cyan median}
	\label{fig:compare_cnn_net_baseline}
\end{figure}

\begin{table}[h]
	\centering
	\begin{tabular}{|l|c|c|}
		\hline
		Dynamic & RNN & CNN \\
		\hline
		All & 0.25 & 0.58 \\
		\OU & 0.27 & 0.48 \\
		Noisy Line & 0.13 & 0.65 \\
		Markovian Switch & 0.41 & 0.64 \\
		\hline
	\end{tabular}
	\caption{Median loss for RNN baseline and CNN estimator for various dynamics on validation set}
	\label{tab:compare_cnn_net_baseline}
\end{table}

%% file: tex/section_model_based.tex
\section{Model based estimators}\label{section:mle}

In this section, we compare the performance of the RNN baseline with classifiers based on maximum likelihood estimation (MLE) of the process parameters. These estimators therefore incorporate knowledge about the underlying data generative process. For each dynamic (see subsection \ref{subsection:tsdynamics}), we compute the MLE estimator of the trend parameter. Then, we use this value at each time step to compute a trend label $\in \{ -1, 0, 1\}$. This approach, which converts a numerical estimate of the trend to a label, is described in the following subsection.\\

In subsections \ref{subsection:mlestimatorsnl}, \ref{subsection:mlestimatorsou} and \ref{subsection:mlestimatorsms} we recall the formulas of the MLE trend estimators and present their empirical performance in comparison with the RNN baseline. Overall, the baseline shows good performance against these estimators. Theoretical details of MLE derivations are included in annex \ref{annex:theorydetails}.

\subsection{From MLE to trend classifier}\label{subsection:mlestimators}

As a reminder, the training data used for the learning step of the neural networks is comprised of piecewise trajectories of the dynamics and uses randomized model parameters. Taking into account this additional randomness in a MLE estimation framework would make the theory intractable. In order to compare MLE based trend classification with neural networks, we use a sliding window mechanism. For a sliding window $W_i$ of length $\eta$:

\begin{itemize}
	\item[-] we compute the value of the trend estimator $\hat{\mu}_i$
	\item[-] we map the value of $\hat{\mu}_i$ to a label using the sign function \footnote{$$
\sgn_\epsilon(x) = 
\begin{cases} 
-1 & x\leq -\epsilon \\
1 & x\geq \epsilon \\
0 & \text{otherwise}
\end{cases}
$$} (for a given threshold $\epsilon$) and predict this label with probability $1$.
\end{itemize}
We only need this mechanism for the Noisy Line Process and the Piecewise \OU\space Process.

\subsection{Noisy Line Estimator}\label{subsection:mlestimatorsnl}
\subsubsection{Derivation of MLE estimator on an interval}

Deriving the maximum likelihood estimator for the slope $\mu$ is easy as any finite sample $(Y_{t_1}, \dots, Y_{t_n})$ on a subdivision $t_1 < \ldots < t_n$ is a Gaussian vector with diagonal covariance matrix. Maximizing the MLE of $\mu$ yields to the slope formula (see annex \ref{annex:noisyline} for mathematical details):

\begin{equation}
\label{noisy-line-mu-est}
\hat{\mu}(y_{t_1}, \ldots, y_{t_n}) = \dfrac{\sum_{i=1}^n (t_i - t_0) (y_{t_i} - y_{t_0})}{\sum_{i=1}^n (t_i - t_0)^2}\text{ .}
\end{equation}

The MLE estimator for the slope follows a normal distribution with mean $\mu$ and variance $\sigma^2 (\sum_{i=1}^n (t_i - t_0)^2)^{-1}$. For a subdivision with constant time step $\delta := t_i - t_{i-1}$ the variance is given by:

$$ \mathrm{V}(\hat{\mu}) =\frac{6 \sigma^2}{n(n+1)(2n+1)}$$

hence decreasing with the number of observations at the rate $n^{-3}$.

\subsubsection{Empirical performance}
Using the same procedure as in section \ref{section:nonmodel}, we compare its performance against our RNN baseline on figure \ref{fig:compare_nl_mle_net_baseline} and table \ref{tab:compare_nl_mle_net_baseline}.
\begin{figure}[h]
               \centering
               \includegraphics[scale=0.5]{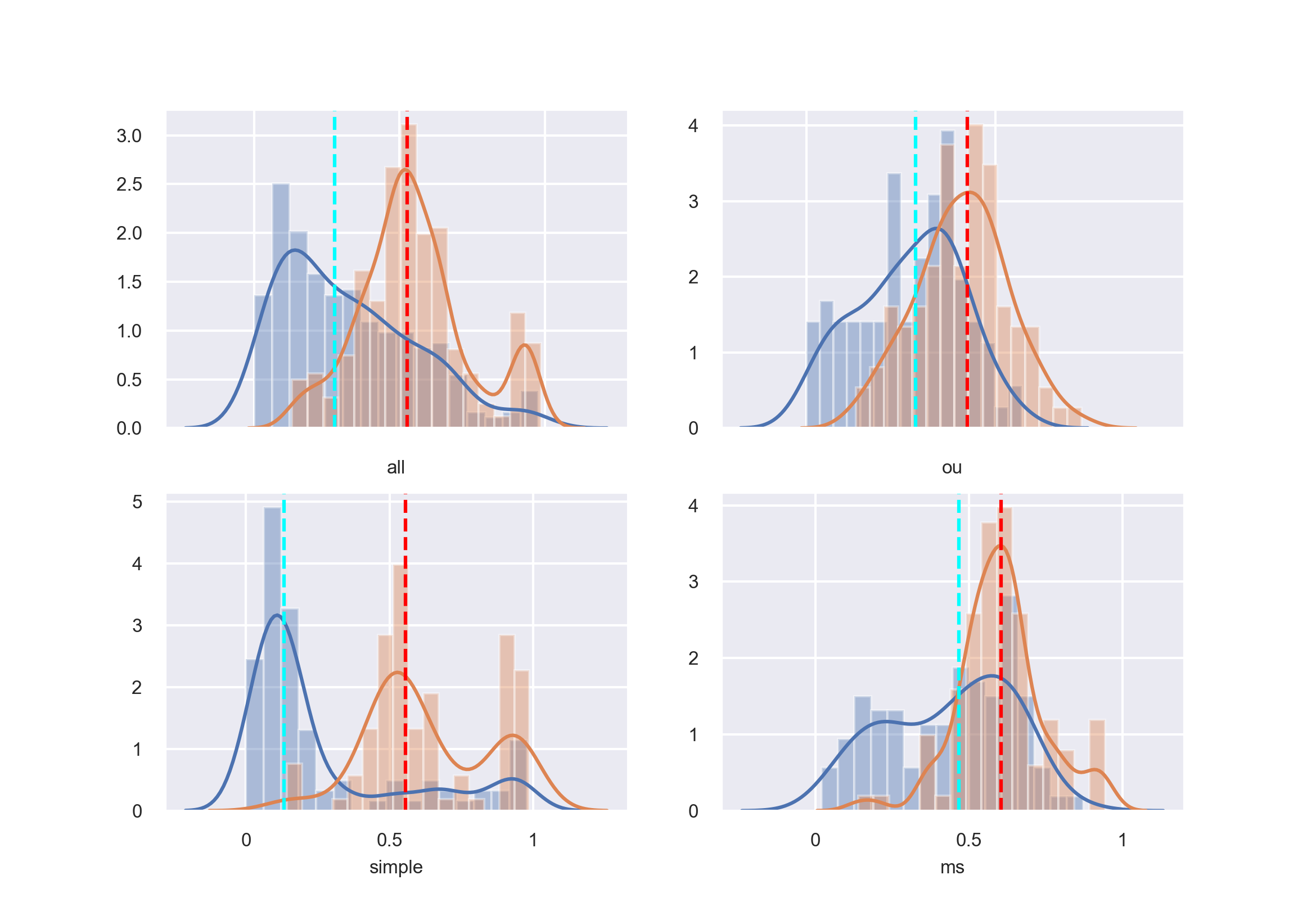}
               \caption{Comparing validation loss distribution for Noisy Line Estimator in orange with red median and RNN baseline in blue with cyan median}
                \label{fig:compare_nl_mle_net_baseline}
\end{figure}

\begin{table}[h]
	\centering
	\begin{tabular}{|l|c|c|}
		\hline
		Dynamic & RNN & NLE \\
		\hline
		All & 0.28 & 0.53 \\
		\OU & 0.29 & 0.42 \\
		Noisy Line & 0.14 & 0.56 \\
		Markovian Switch & 0.47 & 0.61 \\
		\hline
	\end{tabular}
	\caption{Median loss for RNN or Noisy Line Estimator for various dynamics on validation set}
	\label{tab:compare_nl_mle_net_baseline}
\end{table}

The Noisy Line Estimator is easily overtaken by the RNN baseline even on the simple noisy line dynamic\footnote{which is a bit counter-intuitive. The fact that our process is piecewise contrary to the MLE derivation is probably responsible for this underperformance.}.


\subsection{Piecewise OU process}\label{subsection:mlestimatorsou}
\subsubsection{Derivation of MLE estimator on an interval}

Estimating the parameters of time continuous diffusions is a difficult task. One way to construct such estimators is to derive the likelihood function on a discrete grid of prices observations. Due to non-independent samples, likelihood can be hard to derive and its maximisation might require the use of numerical optimization procedures. In the present study we leverage on the theoretical results of \cite{Liptser2013a, Liptser2013b} that express the likelihood function in a simple stochastic integral form. In the case of the \OU\space process with linear trend diffusion:

$$dY_t  = \mu dt - a Y_t dt + \sigma dW_t\text{ ,}$$

the formulas for the estimators are given by:

\begin{equation}
\label{diffusion-line-mu-est}
\hat{\mu} = \dfrac{\frac{1}{2}(Y_T^2 - T) \int_0^T Y_tdt - (Y_T - Y_0) \int_0^T Y_t^2dt}{(\int_0^T Y_t dt)^2 - T \int_0^T Y_t^2dt}
\end{equation}

\begin{equation}
\label{diffusion-line-a-est}
\hat{a} = \dfrac{\frac{1}{2}T (Y_T^2 - Y_0^2 - T)  - (Y_T - Y_0) \int_0^T Y_tdt}{(\int_0^T Y_t dt)^2 - T \int_0^T Y_t^2dt}\text{ .}
\end{equation}

To some extent, an analogy can be drawn with classical OLS estimators $\hat{\beta} = (X^T X)^{-1}X^T y$ where the variance scaling term $(X^TX)^{-1}$ corresponds to the term $\left((\int_0^T Y_t dt)^2 - T \int_0^T Y_t^2dt \right)^{-1}$. The reader can refer to the technical addendum \ref{annex:ou} for mathematical details. When dealing with discrete time observations, the integrals are approximated using the sample values and discrete time increments. Simulations show that these estimators exhibit good empirical properties, although they are biased. It can be shown that the biases for both estimators are given by:

$$ b(\hat{\mu}) = \mathbb{E}_{(\mu, a)} \left[ \frac{(\int_0^T Y_t dW_t) (\int_0^T Y_tdt) - W_T \int_0^T Y_t^2dt}{(\int_0^T Y_t dt)^2 - T \int_0^T Y_t^2 dt }\right] $$ 

$$ b(\hat{a}) = \mathbb{E}_{(\mu, a)} \left[ \frac{T (\int_0^T Y_t dW_t) - W_T \int_0^T Y_t dt}{(\int_0^T Y_t dt)^2 - T \int_0^T Y_t^2 dt }\right] \text{ .}$$ 

In practical applications, the expectations above are computed by first evaluating the residuals $dW_t = dY_t - (\hat{\mu} - \hat{a}Y_t) dt$ over the observed values of $(y_{t_1}, \ldots, y_{t_n})$ and then approximating the integrals by summation of the weighted increments.

\subsubsection{Empirical performance}
We design a trend estimator using the sliding window mechanism of subsection \ref{subsection:mlestimators}. We compare its performance against our RNN baseline on figure \ref{fig:compare_ou_mle_net_baseline} and table \ref{tab:compare_ou_mle_net_baseline}. Interestingly, the performance on the \OU\space dynamic is markedly better and comparable to the performance of the RNN on the \OU\space dynamic.

\begin{figure}[h]
                \centering
               \includegraphics[scale=0.5]{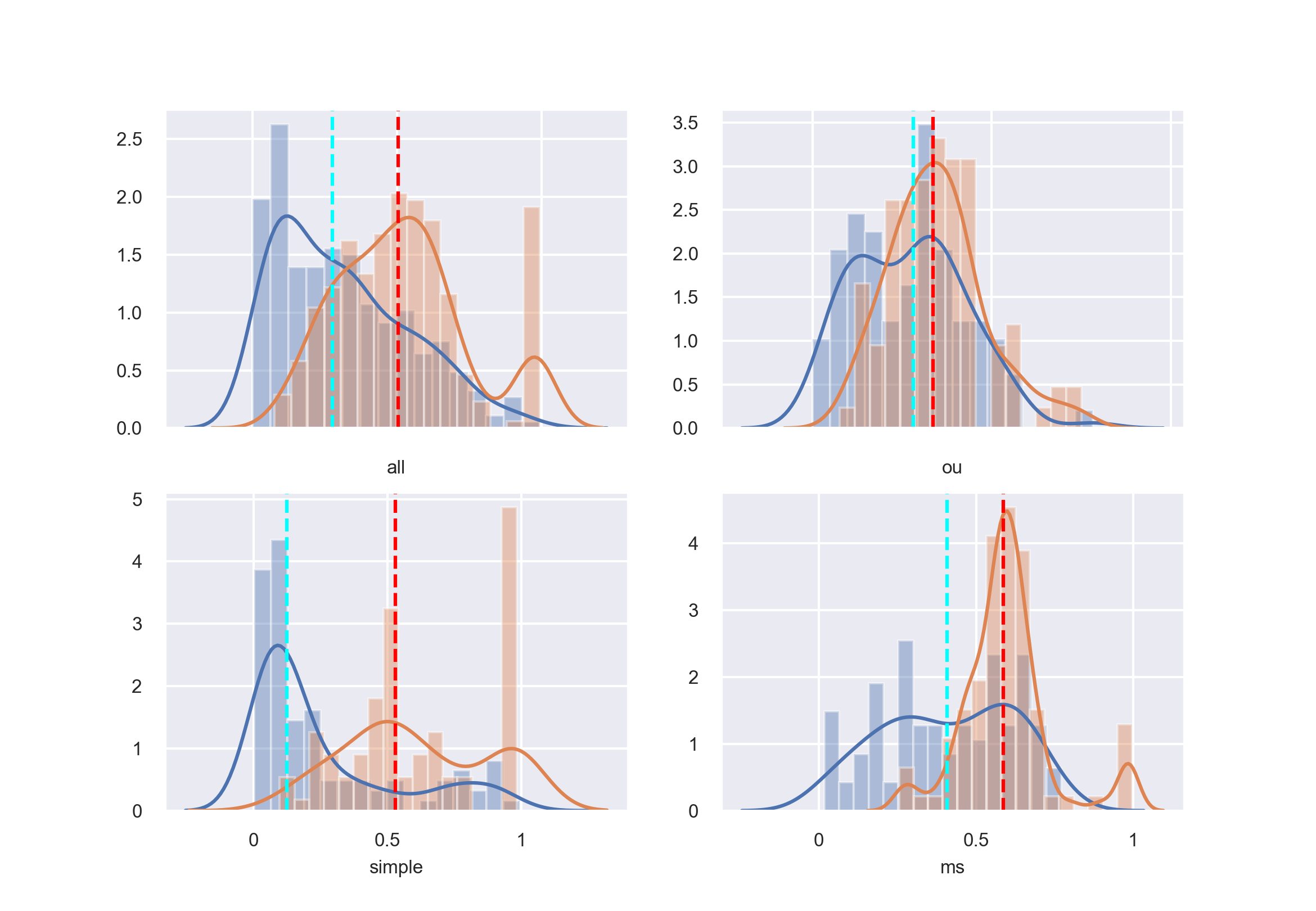}
                \caption{Comparing validation loss distribution for \OU\space Estimator in orange with red median and RNN baseline in blue with cyan median}
                \label{fig:compare_ou_mle_net_baseline}
\end{figure}

\begin{table}[h]
	\centering
	\begin{tabular}{|l|c|c|}
		\hline
		Dynamic & RNN & OUE \\
		\hline
		All & 0.28 & 0.50 \\
		\OU & 0.28 & 0.34 \\
		Noisy Line & 0.12 & 0.53 \\
		Markovian Switch & 0.41 & 0.58 \\
		\hline
	\end{tabular}
	\caption{Median loss for RNN or \OU\space Estimator (OUE) for various dynamics on validation set}
	\label{tab:compare_ou_mle_net_baseline}
\end{table}
\pagebreak

\subsection{Markovian switch process}\label{subsection:mlestimatorsms}

\subsubsection{Derivation of MLE estimator}
The Markovian Switch dynamic described in section \ref{subsubsection:switchingmarkovian} is actually the dynamic of a Hidden Markov Model (HMM) with Gaussian emissions probabilities on log returns:
$$\log\left(\frac{y_{t+1}}{y_t}\right) \sim \mathcal{N}(\gamma \mu_t, \sigma)$$
where $\{\mu_t\}_{t\geq 0}$ is a simple discrete three-state Markov chain. We then use classic techniques (see \cite{Rabiner1989} for example) to get an estimate of the hidden states which have generated $\log\left(\frac{y_{t+1}}{y_t}\right)$.

\subsubsection{Empirical performance}
We train a three-state HMM with Gaussian emission probabilities on the four time series dynamics (as described in subsection \ref{subsection:tsdynamics}). Performance is similar regardless of the training dynamic.
It is not obvious that the hidden states of the HMM will fit in our up, down, flat trend categories. To be able to compute a loss for the HMM, we first map the three-state of the HMM using the mean of the distribution given the hidden state. We sort them in increasing order and map them to down, flat, up states. We would expect to get a sequence of means being negative, close to zero and positive. Actually, only estimators trained on the mixed or Markovian Switch dynamics exhibit means which are clearly separated into a negative, near zero and positive value.
Performance being similar, we use as baseline the estimator trained on the Markovian Switch dynamic which seems the most natural.
Globally, the HMM has a hard time predicting the trend of any dynamic. This might be a bit surprising especially with the Markovian Switch dynamic. We note however that the best validation score is given when the HMM is trained on the Markovian Switch dynamic.
As seen on figure \ref{fig:compare_hmm_net_baseline} or table \ref{tab:compare_hmm_net_baseline} HMM does not provide a good estimator of trend and is easily overtaken by the RNN approach.

\begin{figure}[h!]
	\centering\includegraphics[scale=0.5]{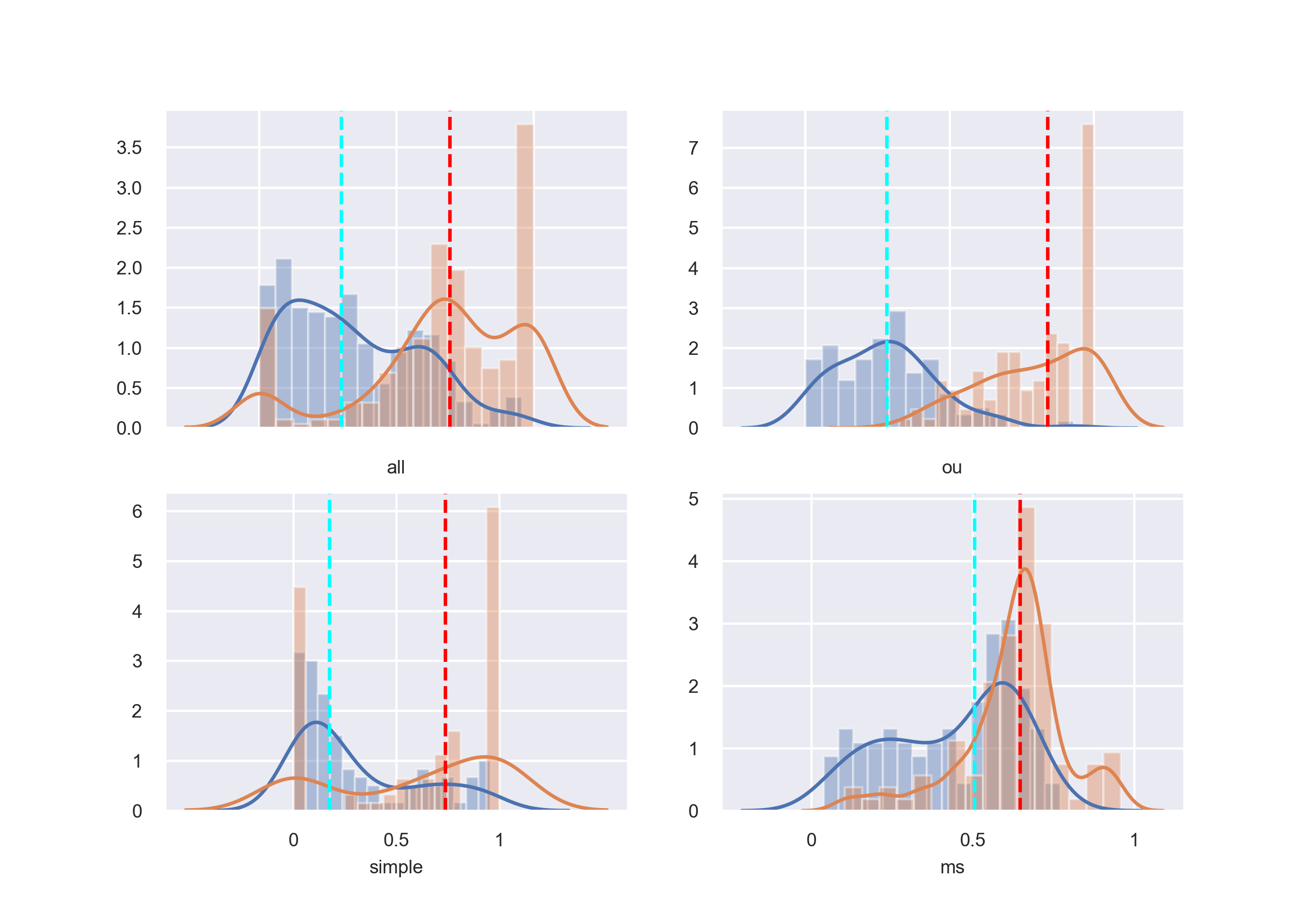}
	\caption{Comparing validation loss distribution for HMM estimator in orange with red median and RNN baseline in blue with cyan median}
	\label{fig:compare_hmm_net_baseline}
\end{figure}

\begin{table}[h]
	\centering
	\begin{tabular}{|l|c|c|}
		\hline
		Dynamic & RNN & HMM \\
		\hline
		All & 0.30 & 0.70 \\
		\OU & 0.28 & 0.84 \\
		Noisy Line & 0.17 & 0.74 \\
		Markovian Switch & 0.50 & 0.64 \\
		\hline
	\end{tabular}
	\caption{Median loss for RNN or HMM estimators for various dynamics on validation set}
	\label{tab:compare_hmm_net_baseline}
\end{table}

%% file: tex/conclusion.tex
\section{Summary}
In this paper, we have investigated the use of several trend estimators on time series behaving similarly to the ones encountered in finance. We have derived theoretical maximum likelihood estimators of trends for two standard dynamics and implemented them. We have shown that certain RNNs are in a way a generalization of simple moving average techniques. For a simple dynamic, we have shown that this generalization transforms the trend estimation problem into locating the state vector. Finally, we have showed empirically that GRU or LSTM cells are on average the best building blocks to use compared to a broad range of estimators in order to detect trends in time series. Putting the emphasis on learning stylized data and then transferring to real data rather than building complex structures fitted to data is also an important takeaway of this paper. Ongoing preliminary research seems to validate our approach for financial applications. This might pave the way to building efficient market estimators protected against over-fitting.

%% file: tex/annex_mletheory.tex
\section{MLE estimators theory}\label{annex:theorydetails}

\subsection{Simple noisy line estimator}\label{annex:noisyline}

On a discrete time grid $t_0 < t_1 < \ldots < t_n$ we consider the ``noisy line'' dynamics:

\begin{equation}
\label{noisy-line-dynamics}
Y_{t_i} = y_{t_0} + \mu (t_i - t_0) + \varepsilon_{t_i}
\end{equation}

where $(\varepsilon_{t_i})_{0 \leq i \leq n}$ is a collection of i.i.d. normal random variables $\mathcal{N}(0, \sigma^2)$.


One can easily show that $(Y_{t_1}, \ldots, Y_{t_n})$ is a Gaussian vector with diagonal covariance matrix. The likelihood function is expressed as

$$ \mathcal{L}(\mu, \sigma^2 \vert y_{t_1}, \ldots y_{t_N} ) = (\sigma \sqrt{2 \pi})^{-n} \times \mathrm{exp} \left( -\frac{1}{2 \sigma^2} \sum_{i=1}^{N} (y_{t_i} - y_{t_0} - \mu  (t_i - t_0))^2  \right)\text{ .}$$
Let $l = \mathrm{log}\mathcal{L}$ denote the log-likelihood. Solving $\dfrac{\partial l}{\partial \mu} = 0$ yields to the expression \eqref{noisy-line-mu-est}.


By expressing $\hat{\mu}$ as

\begin{align*}
\hat{\mu}(Y_{t_1}, \ldots, Y_{t_n}) &= \frac{\sum_{i=1}^n (t_i - t_0) (\mu (t_i - t_0) + \varepsilon_{t_i})}{\sum_{i=1}^n (t_i - t_0)^2}\\
&= \mu + \frac{\sum_{i=1}^n (t_i - t_0) \varepsilon_{t_i}}{\sum_{i=1}^n (t_i - t_0)^2}
\end{align*}

one can show that $\mathbb{E}(\hat{\mu}) = \mu$ and $\mathrm{V}(\hat{\mu}) = \sigma^2 (\sum_{i=1}^n (t_i - t_0)^2)^{-1}$. \\

Simulations of trajectories \eqref{noisy-line-dynamics} to compute samples estimates of $\mu$ are in agreement with the above result.


\subsection{Linear trend with diffusion estimator}\label{annex:ou}

We consider the diffusion with the dynamics

$$dY_t = \mu dt -aY_t dt +  dW_t$$

where $W$ is a Wiener process and $\mu, a$ are unknown scalar quantities to be estimated from observations. In an infinitesimal time period $dt$, the price moves linearly by an amount $\mu dt$ and fluctuates around this trend term by an amount equal to $-a Y_t dt +  dW_t$.\\

We seek to construct estimation techniques for $\mu$ and $a$. In the setting of discrete observations $(y_{t_0}, \ldots, y_{t_i}, \ldots, y_{t_N})$ various estimation approaches can be used. For instance, one can first de-trend the observed price series and then estimate the fluctuation speed $a$ using standard OLS techniques. The drawbacks of such an approach are twofold. Firstly, estimation is conducted regardless of the joint distribution of $(\hat{\mu}, \hat{a})$. Secondly, classical OLS assumptions are most likely to fail in the case of a diffusion price process. As a consequence of non-stationarity of residuals, it can be shown that the OLS estimator of $a$ is biased. Such behaviours are studied in depth in \cite{Yu2012}.\\

Our approach follows the results from \cite{Liptser2013b} in which the authors estimate drift parameters in a continuous likelihood maximization framework. Let us recall the main results from \cite{Liptser2013a, Liptser2013b}.


\begin{theorem}\label{th:equivmeasure}
Let $Y = (Y_t)_{0 \leq t \leq T}$ be a process satisfying the stochastic differential equation (SDE)
	
$$ dY_t = a(t, Y_t)dt + dW_t\, ,\,Y_0 = 0,\,0 \leq t \leq T$$
	
where $a: t\mapsto a(t, .)$ is a non-anticipative function.
	
Under the assumption that  $\mathbb{P}$- almost surely,
	
$$ \int_0^T a(t, Y_t)^2 dt < \infty, \int_0^T a(t, W_t)^2 dt < \infty$$
	
then the measures $\mu_Y$ and $\mu_W$ are equivalent. Moreover, $\mathbb{P}$-almost surely, the Radon-Nikodym derivative of $\mu_Y$ with respect to $\mu_W$ is given by:
\begin{equation}
\label{radon-nykodim-deriv}
\frac{d\mu_Y}{d\mu_W}(t, Y_t) = \mathrm{exp}\left( \int_0^t a(s, Y_s) dY_s -\dfrac{1}{2} \int_0^t a(s, Y_s)^2 ds \right).
\end{equation}
\end{theorem}


The reader can refer to \cite{Liptser2013a}, Theorem 7.7, for a formal statement and proof. The issue of the drift parametric estimation is addressed in \cite{Liptser2013b} by considering the diffusion process:

\begin{equation}
\label{diffusiononedrift}
dY_t = \theta \alpha(t, Y_t) dt + dW_t\text{ .}
\end{equation}

Using the result above with $a(t,x) = \theta \alpha(t, x)$ and under similar assumption on $\alpha$ one can show that the measures $\mu_Y^\theta$ and $\mu_W$ are equivalent and that the likelihood function $\mathcal{L}_\theta(Y)$ can be expressed as

$$ \mathcal{L}_\theta (Y) =\mathrm{exp} \left( \theta \int_0^t \alpha(s, Y_s)dY_s - \dfrac{\theta^2}{2}\int_0^t \alpha(s, Y_s)^2 ds \right)\text{ .}$$

It is easy to show that the log-likelihood is a concave function of the parameter $\theta$ and that its maximum is attained for $\theta^*$ such that $\dfrac{\mathcal{L}_\theta}{d \theta}(\theta^*) = 0$.

%
%
%
%
%

As a consequence, under the assumption that
\begin{equation}
\label{one-drift-conditions}
\int_0^T \alpha(t, Y_t)^2 dt < \infty, \int_0^T \alpha(t, W_t)^2 dt < \infty
\end{equation}

and under the condition that $\mathbb{P}_\theta$-a.s. $\int_0^T \alpha(t, Y_t) dt > 0$ the maximum likelihood estimation of $\hat{\theta}(Y)$ is expressed as:

\begin{equation}
\label{mle-est-1d-formula}
\hat{\theta}(Y) = \dfrac{\int_0^T \alpha(t, Y_t) dY_t}{\int_0^T \alpha(t, Y_t)^2 dt}.
\end{equation}

When dealing with real data, the numerical value of $\hat{\theta}$ is computed using numerical integration techniques along the observed path $(y_{t_0}, \ldots, y_{t_N})$. From now on, we adopt the lighter notations:

$$I_Y(\alpha) := \int_0^T \alpha(t, Y_t) dY_t $$

$$ I_t (\alpha) := \int_0^T \alpha(t, Y_t) dt$$

so that the MLE estimator \eqref{mle-est-1d-formula} is expressed as $\dfrac{I_Y(\alpha)}{I_t(\alpha^2)}$.


For most drift functions $\alpha$ the estimator $\hat{\theta}$ has non-zero bias. An approximation of the bias can be easily derived by substituting the expression of $d Y_t$  in \eqref{mle-est-1d-formula}:

\begin{align*}
\hat{\theta}(Y)  &=  \dfrac{\int_0^T \alpha(t, Y_t) dY_t}{\int_0^T \alpha(t, Y_t)^2 dt}\\
&= \dfrac{\int_0^T \alpha(t, Y_t)(\theta \alpha(t, Y_t) dt + dW_t)}{\int_0^T \alpha(t, Y_t)^2 dt}\\
&= \theta + \dfrac{\int_0^T \alpha (t, Y_t) dW_t}{\int_0^T \alpha(t, Y_t)^2 dt}
\end{align*}

Hence the bias $b(\hat{\theta}(Y) )= \mathbb{E}_\theta (\hat{\theta} - \theta)$ can be computed by approximating the expectation:

$$ \mathbb{E}_\theta \left(\dfrac{\int_0^T \alpha (t, Y_t) dW_t}{\int_0^T \alpha(t, Y_t)^2 dt} \right)\text{ .}$$

In the following, we extend \eqref{diffusiononedrift} to the 2D parametric drift case:

\begin{equation}
\label{diffusiontwodrifts}
d Y_t = (\theta_1 \alpha_1(t, Y_t) + \theta_2 \alpha_2(t, Y_t))dt + dW_t.
\end{equation}

\begin{theorem}\label{mle-est-2d}

Let $(Y_t)_{t \geq 0}$ be a process satisfying the diffusion equation \eqref{diffusiontwodrifts} where both $\alpha_1$ and $\alpha_2$ satisfy the condition \eqref{one-drift-conditions}.

Under the condition that $\mathbb{P}_\theta$-a.s. $\int_0^T \alpha_i(t, Y_t) dt > 0, i=1, 2$ the maximum likelihood estimation of $\hat{\theta}(Y)$ is expressed as:

\begin{equation}
\label{mle-est-2d-formula-1}
\hat{\theta}_1(Y) = \dfrac{I_Y(\alpha_2) I_t(\alpha_1 \alpha_2) - I_Y(\alpha_1) I_t(\alpha_2^2)}{I_t(\alpha_1 \alpha_2)^2 - I_t(\alpha_1^2) I_t(\alpha_2 ^2)},
\end{equation}

\begin{equation}
\label{mle-est-2d-formula-2}
\hat{\theta}_2(Y) = \dfrac{I_Y(\alpha_1) I_t(\alpha_1 \alpha_2) - I_Y(\alpha_2) I_t(\alpha_1^2)}{I_t(\alpha_1 \alpha_2)^2 - I_t(\alpha_1^2) I_t(\alpha_2 ^2)}.
\end{equation}
\end{theorem}

\begin{proof}
By substituting the drift term in \eqref{diffusiontwodrifts} into \eqref{radon-nykodim-deriv} one obtains

$$ \dfrac{d \mu_Y}{d \mu_W} (Y)  = \mathrm{exp}\left(\theta_1 I_Y(\alpha_1) + \theta_2 I_Y(\alpha_2)  - \dfrac{\theta_1^2}{2} I_t(\alpha_1^2) - \dfrac{\theta_2^2}{2} I_t(\alpha_2^2)  - \theta_1 \theta_2 I_t(\alpha_1 \alpha_2)\right)$$

Let $l_\theta(Y)$ denote the log-likelihood. To ensure the concavity of $l_\theta$ one must verify that its Hessian matrix $H = (\partial_{i, j} l_\theta), 1 \leq i, j \leq 2$ is definite negative.

Deriving the Hessian yields to 

$$ H = -  \begin{pmatrix} I_t(\alpha_1^2) & I_t(\alpha_1 \alpha_2) \\
I_t(\alpha_1 \alpha_2) & I_t(\alpha_2^2)

\end{pmatrix}$$

hence of the form $$ H = - \begin{pmatrix} A & C \\ C & B \end{pmatrix}.$$

The eigenvalues of $H$ are given by 

$$ \lambda_1 = -\dfrac{1}{2} (A + B + \sqrt{(A - B)^2 + 4C^2}) < 0\text{ ,}$$ 

$$ \lambda_2 = \dfrac{1}{2} (- A - B + \sqrt{(A - B)^2 + 4C^2})\text{ .}$$  

For its largest eigenvalue $\lambda_2$ to be negative is equivalent to $C^2 < AB$, that is $I_t(\alpha_1 \alpha_2)^2 < I_t(\alpha_1^2) I_t(\alpha_2^2)$. This latter expression is equivalent to the Cauchy-Schwartz inequality. Hence these conditions are $\mathbb{P}_\theta$-a.s. verified, ensuring the concavity of $l_\theta$. Finally one can deduce the equations \eqref{mle-est-2d-formula-1} and \eqref{mle-est-2d-formula-2} by solving the first order conditions $\partial_i l_\theta (\theta_1^*, \theta_2^*) = 0,  i= 1, 2$

\end{proof}


We now consider the diffusion:

\begin{equation}
\label{diffusion-line-dynamics}
dY_t  = \mu dt - a Y_t dt + dW_t
\end{equation}

From the results above the MLE estimators for both $\mu$ and $a$ are given by:

\begin{equation}
\label{diffusion-line-mu-est}
\hat{\mu} = \dfrac{\frac{1}{2}(Y_T^2 - T) \int_0^T Y_tdt - (Y_T - Y_0) \int_0^T Y_t^2dt}{(\int_0^T Y_t dt)^2 - T \int_0^T Y_t^2dt}
\end{equation}

\begin{equation}
\label{diffusion-line-a-est}
\hat{a} = \dfrac{\frac{1}{2}T (Y_T^2 - Y_0^2 - T)  - (Y_T - Y_0) \int_0^T Y_tdt}{(\int_0^T Y_t dt)^2 - T \int_0^T Y_t^2dt}
\end{equation}

To obtain these formulas we use the formulas \eqref{mle-est-2d-formula-1} and \eqref{mle-est-2d-formula-2} with $\alpha_1 (t, x) = 1$, $\theta_1 = \mu$, $\alpha_2(t, x) = -x$ and $\theta_2 = a$. Using Ito's Lemma one can show that:

$$ I_Y(\alpha_1) =  \int_0^T dY_t = Y_T - Y_0$$
$$ I_Y(\alpha_2) = -\int_0^T Y_tdY_t = \frac{1}{2}(T - Y_T^2+ Y_{0}^2)$$
$$I_t(\alpha_1 \alpha_2) = - \int_0^T Y_t dt$$
$$I_t(\alpha_1^2) = T$$
$$I_t(\alpha_2^2) = \int_0^T Y_t^2dt.$$

%% file: tex/annex_convexcell.tex
\section{Asymptotic state behaviour in a simple case}\label{annex:convexcell}
We prove in this annex the results stated in the worked example of section \ref{subsection:motivation}.
We consider the following process\footnote{the state of a vanilla RNN with identity activation function, no biases and constrained weights}
$$
h_t = W_{hh}\, h_{t-1} + Y_t \, w_{ih}
$$
where $W_{hh} \in \mathcal{M}_{n}(\R)$ and $w_{ih} \in \R^n$ and
\begin{eqnarray*}
\forall i, j \in \llbracket 1, n \rrbracket \quad (W_{hh})_{(i, j)} &>& 0 \\
\forall i \in \llbracket 1, n \rrbracket \quad (w_{ih})_{i} &>& 0 \\
\|W_{hh}\|_1 = \max_{1 \leq j \leq n} \sum_{i=1}^m | W_{hh})_{(i, j)} | &<&1\\
\|W_{hh}\|_\infty = \max_{1 \leq i \leq m} \sum _{j=1}^n | W_{hh})_{(i, j)} | &<&1
\end{eqnarray*}

Let's consider $Y_t$ a simple noisy line process $Y_t = Y_0 + \mu \, t + \epsilon_t$ we have:
$$
h_t = W_{hh}\, h_{t-1} + \mu\, t \, w_{ih} + \epsilon_t \, w_{ih} + Y_0 \, w_{ih}
$$
$\mu$ being the trend and $\{\epsilon_t\}_{t\geq 0}$ an i.i.d noise process with expectation equal to zero and unit variance.\\
\textbf{Without trend i.e. $\mu=0$},  we have
$$
h_{t} = W_{hh}\, h_{t-1} + Y_0 \, w_{ih} + \epsilon_t \, w_{ih}\text{ .}
$$
We note $\lambda_{pf}>0$ the Perron-Frobenius eigenvalue of $W_{hh}$. All eigenvalues of $W_{hh}$ different from $\lambda_{pf}$ satisfy $|\lambda| < \lambda_{pf}$. If $y=(y_i)>0$ is a corresponding eigenvector then 
$$
\forall i \in \llbracket 1, n \rrbracket \quad \sum_{j} W_{ij} y_j = \lambda_{pf} y_i\text{ .}
$$
Noting $y_\text{max}=\max(\{y_1, y_2, \ldots, y_n\})$ 
$$
\forall i \in \llbracket 1, n \rrbracket \quad  \lambda_{pf} y_i = \sum_{j} W_{ij} y_j \leq \sum_{j} W_{ij} y_\text{max} < y_\text{max}\text{ .}
$$
So $\lambda_{pf} < 1$ and $I-W_{hh}$ is bijective. We can define $h^{*}$ by
$$
(I-W_{hh})h^{*} = Y_0 w_{ih}
$$
Then,
$$
h_{t} - h^{*} = W_{hh}\, (\hat{h}_{t-1} - h^{*}) + \epsilon_t w_{ih}\text{ .}
$$
Simplifying notations with $\overset{\sim}{h}_t = h_{t} - h^{*}$, $W = W_{hh}$ and $\overset{\sim}{\epsilon}_t = \epsilon_t w_{ih}$
$$
\forall t\in \llbracket 1, n \rrbracket\quad
\overset{\sim}{h}_{t+1} = W^t \overset{\sim}{h}_{1} + \sum_{k=0}^{t-1} W^{t-k} \overset{\sim}{\epsilon}_k
$$
where $\overset{\sim}{\epsilon}_t$ has zero mean and variance equal to $\omega= w_{ih} w_{ih}\tran$.\\
Defining
$$
R_t = \sum_{k=1}^{t} W^{t-k} \overset{\sim}{\epsilon}_k\text{ ,}
$$
$\E(R_t)=0$ and $V(R_t) = \sum_{k=0}^{t-1} W^{k} \omega (W\tran )^{k}$ which is absolutely summable as
$$
\|W^{k} \omega (W\tran )^{k}\|_{\infty} \leq \|\omega\| (\|W\|_{\infty}\|W\|_{1})^k\text{ .}
$$
$(W^t \overset{\sim}{h}_{1})_{t\in \N}$ converges almost surely to 0. As $t$ goes towards infinity $h_t$ becomes centered around a random variable of finite variance.


\textbf{With a trend i.e. $\mu \neq 0$}, noting $h^0$ the previous no trend solution and $h^{\mu}$ the process with a trend $\mu$, and defining
$$
\delta_t = h^{\mu}_t - h^{0}_t
$$
it is easily seen that
$$
\delta_{t} = W_{hh}^{t}\delta_{0} + W_{hh}\, \delta_{t-1} + \mu\, t \, w_{ih}\text{ .}
$$
Noting that $\delta_{0}=0$ 
$$
\delta_{t} = \mu  \sum_{k=0}^{t} k \underbrace{W^{t-k}_{hh} \, w_{ih}}_{P_{k, t}}
$$
$P_{k, t}\geq 0$, so if $\mu>0$
$$
\delta_{t} \geq \mu\,t\,w_{ih} \underset{t\to +\infty}{\longrightarrow}+\infty\text{ .}
$$
Similarly, if $\mu <0$,
$$
\delta_{t} \leq \mu\,t\,w_{ih} \underset{t\to +\infty}{\longrightarrow}-\infty\text{ .}
$$
So,
$$
h^{\mu}_t = \underset{\text{finite variance}}{h^{0}_t}+ \underset{\text{diverge almost surely}}{\delta_t}\text{ .}
$$

%% file: tex/annex_hidden_state.tex
\section{Visual representation of hidden state}\label{annex:hiddenstateplot}
We plot the hidden state $h_t\in\R^5$ of a vanilla network previously trained on a randomly chosen dynamic. The hidden state is obtained running through three different Noisy Line Processes (respectively up trending, without trend and down trending). We see, on figure \ref{fig:show_hidden_convex_cell}, that the state goes right as time goes for the down trend, stays around zero without trend and goes left for the uptrend. The state has been projected into the plane of the first two eigenvectors to get a two dimensional plot.
\begin{figure}[h]
	\centering
	\includegraphics[scale=0.5]{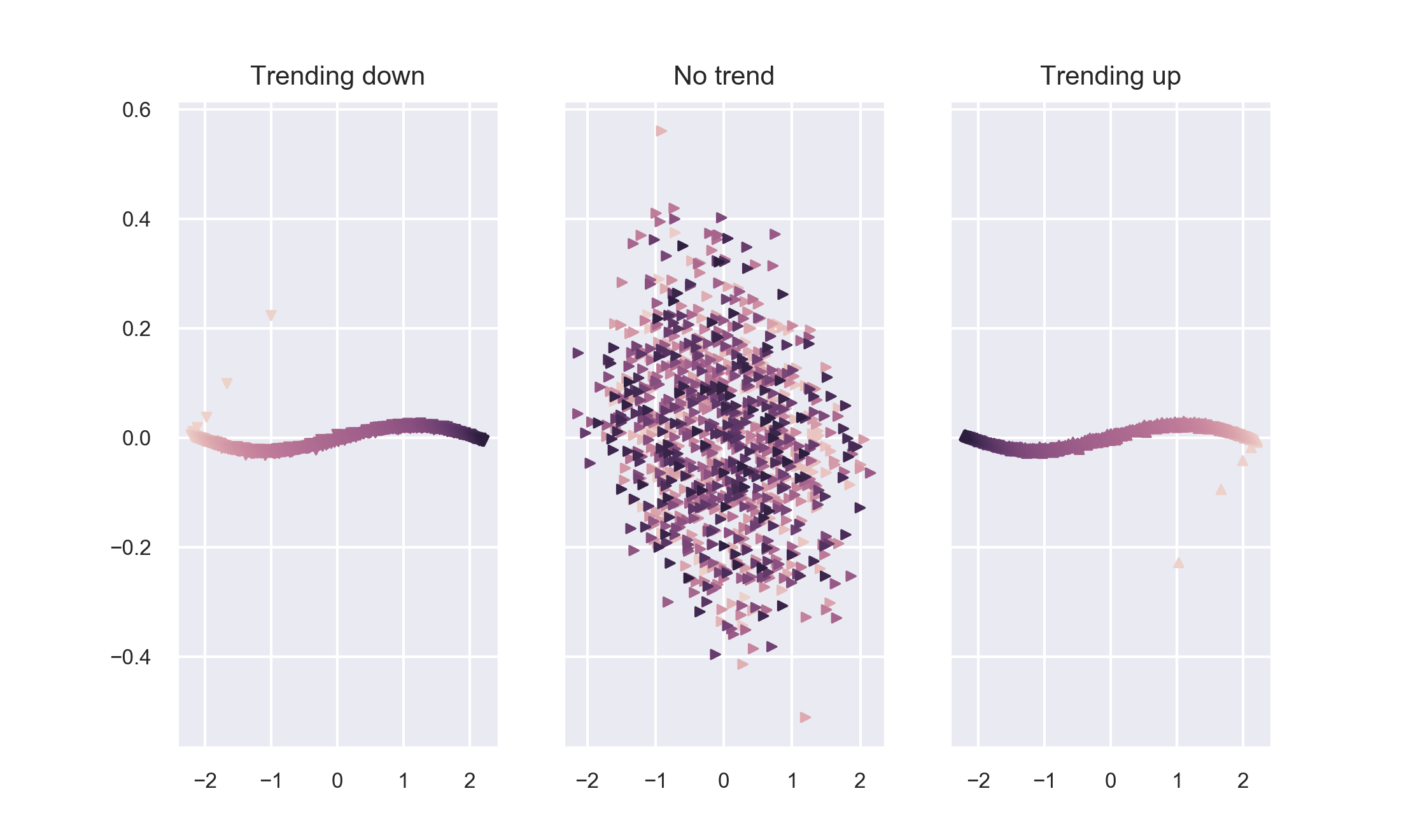}
	\caption{Hidden state for up, no and down trend. Colour goes darker with time}
	\label{fig:show_hidden_convex_cell}
\end{figure}

%% file: tex/annex_rnn_training_details.tex
\section{Technical details}\label{annex:rnntraining}

\subsection{RNN training details}\label{annex:rnntrainingdetails}
\begin{table}[h]
\begin{subfigure}[b]{\textwidth}
\centering
\begin{tabular}{|l|r|}
	\hline
	RNN type & Vanilla, LSTM, GRU \\
	Number of layers & 1, 2\\
	Learning rate & 0.01, 0.1, 1.0 \\
	Dropout & 0, 0.1 \\
	\hline
\end{tabular}
\caption{Training hyper-parameters for RNNs}
\label{tab:training_schedule_rnn}
\end{subfigure}
\\
\begin{subfigure}[b]{\textwidth}
	\centering
	\begin{tabular}{|l|r|}
		\hline
		Time-series dynamic& Piecewise Noisy Line\\
		&Piecewise \OU \\
		&Markovian Switch\\
		&Mixed Dynamic\\
		Global ``noise level'' & 1, 5 \\
		Number of samples & 1,000 \\
		\hline
	\end{tabular}
	\caption{Training hyper-parameters for time series dynamics}
	\label{tab:training_schedule_ts}
\end{subfigure}
\\
\begin{subfigure}[b]{\textwidth}
	\centering
\begin{tabular}{|l|c|c|}
	\hline
	Dynamic & Min Length & Max Length \\
	\hline
	Piecewise Noisy Line & 50 & 1000 \\
	Piecewise \OU  & 80 & 2400 \\
	Markovian Switch & 500 & 1000 \\
	Mixed Dynamic & 1000& 1000\\
	\hline
\end{tabular}
	\caption{Sequence lengths}
\label{tab:training_schedule_len_seq}
\end{subfigure}
\end{table}
\pagebreak

\subsection{RNN empirical findings}\label{annex:rnnempiricalfindings}
\begin{table}[h]
	\centering
	\scalebox{0.7}{
	\begin{tabular}{|l|c|c|c|c|c|c|}
		\hline
		Feature[Modality]&Coefficient&Std Err&t-statistic&P-value&\multicolumn{2}{|c|}{5\% confidence interval}\\
		\hline
Intercept&0.4840&0.002&260.514&0.000&0.480&0.488\\
Training dynamic[Markovian Switch]&0.0066&0.002&3.852&0.000&0.003&0.010\\
Training dynamic[\OU]&0.0290&0.002&16.880&0.000&0.026&0.032\\
Training dynamic[Noisy Line]&0.0017&0.002&1.002&0.316&-0.002&0.005\\
Net type[LSTM]&0.0366&0.002&24.416&0.000&0.034&0.040\\
Net type[Vanilla]&0.1742&0.002&116.118&0.000&0.171&0.177\\
Optimization[RMSP]&0.0234&0.001&19.489&0.000&0.021&0.026\\
Testing[\OU]&-0.1006&0.001&-68.400&0.000&-0.103&-0.098\\
Testing[Noisy Line]&-0.0357&0.001&-24.271&0.000&-0.039&-0.033\\
		\hline
	\end{tabular}}
	\caption{Loss OLS left hand column is the feature column with the specified modality in bracket}
	\label{tab:annexsimpleols}
\end{table}

\subsection{Vanilla structure no better than a dummy predictor}\label{annex:vanillaisdummy}
\begin{figure}[h]
	\centering
	\includegraphics[scale=0.4]{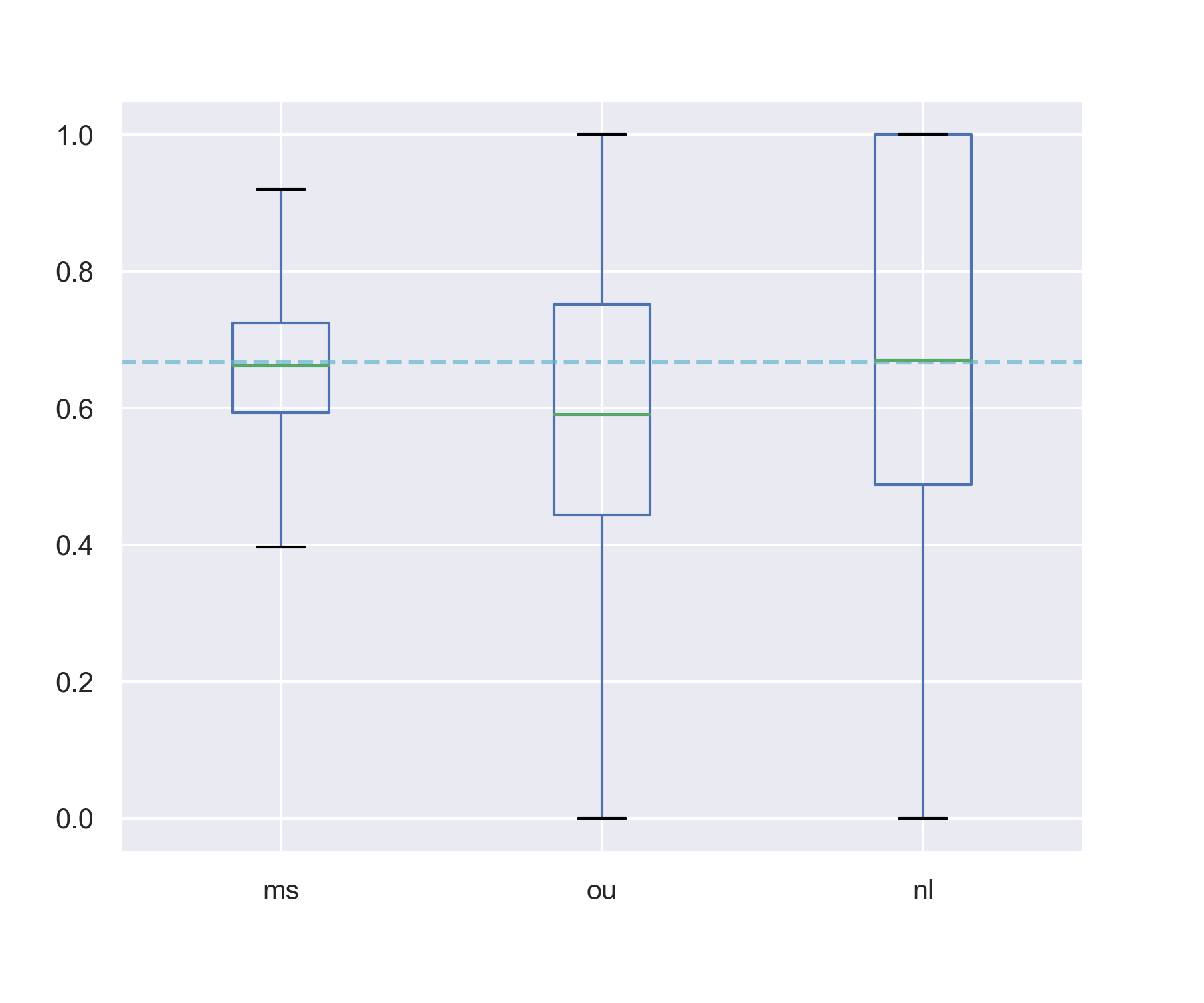}
	\caption{Loss of Vanilla Structure depending on the test dynamic category. nl is Noisy Line, ou \OU\space Process and ms Markovian Switch. Cyan dashed line is the average loss of a dummy classifier}
	\label{fig:vanillaisdummyplot}
\end{figure}
\pagebreak